\documentclass{article}
\usepackage[utf8]{inputenc}

\usepackage{amsmath}
\usepackage{amssymb}
\usepackage{amsthm}
\usepackage{graphicx}
\usepackage{algorithm}
\usepackage{algcompatible}
\usepackage{natbib}
\usepackage[utf8]{inputenc}
\usepackage[english]{babel}
\usepackage[toc,page]{appendix}
\usepackage{hyperref}
\usepackage[section]{placeins}
\usepackage{natbib}
\newtheorem{innercustomthm}{Theorem}

\newtheorem{theorem}{Theorem}[section]
\newtheorem{lemma}[theorem]{Lemma}

\DeclareMathOperator*{\minimise}{minimize}
\DeclareMathOperator*{\subject}{subject}

\begin{document}
	
\title{Informed Non-convex Robust Principal Component Analysis with Features}

\author{Niannan Xue, Jiankang Deng, Yannis Panagakis, Stefanos Zafeiriou\\
Department of Computing, Imperial College London, UK, SW7 2AZ\\
}
\date{}

\maketitle

\begin{abstract}
We revisit the problem of robust principal component analysis with features acting as prior side information. To this aim, a novel, elegant, non-convex optimization approach is proposed to decompose a given observation matrix into a low-rank core and the corresponding sparse residual.
Rigorous theoretical analysis of the proposed algorithm results in exact recovery guarantees with low computational complexity. Aptly designed synthetic experiments demonstrate that our method is the first to wholly harness the power of non-convexity over convexity in terms of both recoverability and speed. That is, the proposed non-convex approach is more accurate and faster compared to the best available algorithms for the problem under study.
Two real-world applications, namely image classification and face denoising further exemplify the practical superiority of the proposed method.
\end{abstract}

\section{Introduction}
\noindent Many machine learning and artificial intelligence tasks involve the separation of a data matrix into a low-rank structure and a sparse part capturing different information. Robust principal component analysis (RPCA) \citet{candes11,chandrasekaran11} is a popular framework that logically characterizes this matrix separation problem.\\
\indent Nevertheless, prior side information, oftentimes in the form of features, may also be present in practice. For instance, features are available for the following tasks:\par
\renewcommand{\labelitemi}{$\textendash$}
\begin{itemize}
   \item  Collaborative filtering: apart from ratings of an item by other users, the profile of the user and the description of the item can also be exploited in making recommendations \citet{chiang15};
   \item Relationship prediction: user behaviours and message exchanges can assist in finding missing links on social media networks \citet{xu13};
   \item Person-specific facial deformable models: an orthonormal subspace learnt from manually annotated data captured in-the-wild, when fed into an image congealing procedure, can help produce more correct fittings \citet{sagnoas14}.
 \end{itemize}
It is thus reasonable to investigate how propitious it is for RPCA to exploit the available features. Indeed, 
recent results \citet{liu17} indicate that features are not redundant at all. In the setting of multiple subspaces, RPCA degrades as the number of subspaces grows because of the increased row-coherence. On the other hand, the use of feature dictionaries allows accurate low-rank recovery by removing the dependency on row-coherence. Despite the theoretical and practical merits of convexified RPCA with features, such as LRR \citet{liu10} and PCPF \citet{chiang16}, convex relaxations of the rank function and $l_0$-norm necessarily lead into \emph{algorithm weakening} \citet{chandrasekarana13}.\\
\indent On a separate note, recent advances in non-convex optimization algorithms continue to undermine their convex counterparts \citet{gong13,ge16,kohler17}. In particular, non-convex RPCA algorithms such as fast RPCA \citet{yi16} and AltProj \citet{netrapalli14} exhibit better properties than the convex formulation. Most recently, \citet{niranjan17} embedded features into a non-convex RPCA framework known as IRPCA-IHT with faster speed. However, it remains unclear as to whether features have been effectively incorporated into non-convex RPCA and the benefits of accuracy, speed and so on have been exploited as much as possible.\\
\indent In this work, we give positive answers to the above questions by proposing a novel, non-convex scheme that fully leverages features, which reveal true row and column subspaces, to decompose an observation matrix into a core matrix with given rank and a residual part with informed sparsity. Even though the proposed algorithm is inspired by the recently proposed fast RPCA \citet{yi16}, our contributions are by no means trivial, especially from a theoretical perspective. First, fast RPCA cannot be easily extended to consistently take account of features. Second, as we show in this paper, incoherence assumptions on the observation matrix and features play a decisive role in determining the corruption bound and the computational complexity of the non-convex algorithm. Third, fast RPCA is limited to a corruption rate of $50\%$ due to their choice of the hard threshold, whereas our algorithm ups this rate to $90\%$. Fourth, we prove that the costly projection onto factorized spaces is entirely optional when features satisfy certain incoherence conditions. Although our algorithm maintains the same corruption rate of $O(\frac{n}{r^{1.5}})$ and complexity of $O(rn^2\log(\frac{1}{\epsilon}))$ as fast RPCA, we show empirically that massive gains in accuracy and speed can still be obtained. Besides, the transfer of coherence dependency from observation to features means that our algorithm is capable of dealing with highly incoherent data.\\
\indent Unavoidably, features adversely affect tolerance to corruption in IRPCA-IHT ($O(\frac{n}{d})$) compared to its predecessor AltProj ($O(\frac{n}{r})$). This is not always true with our algorithm in relation to fast RPCA. And when the underlying rank is low but features are only weakly informative, \textit{i.e.} $r\ll d$, which is often the case, our tolerance to corruption is arguably better. IRPCA-IHT also has a higher complexity of $O((dn^2 + d^
2r)\log(\frac{1}{\epsilon}))$ than that of our algorithm. Although feature-free convex and non-convex algorithms have higher asymptotic error bounds than our algorithm, we show in our experiments that this does not translate as accuracy in reality. Our algorithm still has the best performance in recovering accurately the low-rank part from highly corrupted matrices. This may be attributed to the fact that our bounds are not tight. Besides, PCPF and AltProj have much higher complexity ($O(\frac{n^3}{\sqrt{\epsilon}})$ and $O(r^2n^2\log(\frac{1}{\epsilon}))$) than ours. For PCPF, there does not exist any theoretical analysis under the deterministic sparsity model. Nonetheless, we show in our experiments that our algorithm is superior with regard to both recoverability and running time. The overall contribution of this paper is as follows:
\renewcommand{\labelitemi}{$\bullet$}
\begin{itemize}
   \item  A novel non-convex algorithm integrating features with informed sparsity is proposed in order to solve  RPCA problem.
   \item We establish theoretical guarantees of exact recovery under different assumptions regarding the incoherence of features and observation.
   \item Extensive experimental results on synthetic data indicate that the proposed algorithm is faster and more accurate in low-rank matrix recovery than the compared state-of-the-art convex and non-convex methods for RPCA (with and without features).
   \item Experiments on two real-world datasets, namely MNIST and Yale B database demonstrate the practical merits of the proposed algorithm.
 \end{itemize}

\section{Notations}
Lowercase letters denote scalars and uppercase letters denote matrices, unless otherwise stated. $\mathbf{A}i\cdot$ and $\mathbf{A}\cdot j$ represent the $i^\text{th}$ row and the $j^\text{th}$ column of $\mathbf{A}$. Projection onto support set $\Omega$ is given by $\mathbf{\Pi}_\Omega$. $|\mathbf{A}|$ is the element-wise absolute value of matrix $\mathbf{A}$. For norms of matrix $\mathbf{A}$, $\Vert \mathbf{A}\Vert_F$ is the Frobenius norm; $\Vert \mathbf{A}\Vert_*$ is the nuclear norm; $\Vert\mathbf{A}\Vert_2$ is the largest singular value; otherwise, $\Vert\mathbf{A}\Vert_p$ is the $l_p$-norm of vectorized $\mathbf{A}$; and $\Vert \mathbf{A}\Vert_{2,\infty}$ is the maximum of matrix row $l_2$-norms. Moreover, $\langle \mathbf{A},\mathbf{B}\rangle$ represents tr($\mathbf{A}^T\mathbf{B}$) for real matrices $\mathbf{A},\mathbf{B}$. Additionally, $\sigma_i$ is the $i^\text{th}$ largest singular value of a matrix.\\
\indent The Euclidean metric is not applicable here because of the non-uniqueness of the bi-factorisation $\mathbf{L}^*=\mathbf{A}^*\mathbf{B}^{*T}$, which corresponds to a manifold rather than a point. Hence, we define the following distance between $(\mathbf{A},\mathbf{B})$ and any of the optimal pair $(\mathbf{A}^*,\mathbf{B}^*)$ such that $\mathbf{L}^*=\mathbf{A}^*\mathbf{B}^{*T}$:
\begin{equation}
	d(\mathbf{A},\mathbf{B},\mathbf{A}^*,\mathbf{B}^*)=\min_{\mathbf{R}}\sqrt{\Vert \mathbf{A}-\mathbf{A}^*\mathbf{R}\Vert_F^2+\Vert\mathbf{B}-\mathbf{B}^*\mathbf{R}\Vert_F^2},
\end{equation}
where $\mathbf{R}$ is an $r\times r$ orthogonal matrix.

\section{Related Work}
RPCA concerns a known observation matrix $\mathbf{M}$ which we are seeking to decompose into matrices $\mathbf{L}^*$, $\mathbf{S}^*$ such that $\mathbf{L}^*$ is low-rank and $\mathbf{S}^*$ is sparse and of arbitrary magnitude. Conceptually, it is equivalent to solving the following optimization problem:  
\begin{equation}
\label{eq:rpca}
    \min_{\mathbf{L},\mathbf{S}}\ \ \text{rank}(\mathbf{L})+\gamma\Vert\mathbf{S}\Vert_0\quad\text{subject to}\quad\mathbf{L}+\mathbf{S}=\mathbf{M},
\end{equation}
for appropriate $\gamma$. This problem, regrettably, is NP-hard.\\
\indent PCP \citet{wright09} replaces (\ref{eq:rpca}) with convex heuristics:
\begin{equation}
\label{eq:pcp}
    \min_{\mathbf{L},\mathbf{S}}\ \ \Vert\mathbf{L}\Vert_*+\gamma\Vert\mathbf{S}\Vert_1\quad\text{subject to}\quad\mathbf{L}+\mathbf{S}=\mathbf{M},
\end{equation}
for some $\gamma$. In spite of the simplification, PCP can exactly recover the solution of RPCA under the random model \citet{candes11} and the deterministic model \citet{chandrasekaran11,hsu11}.\\
\indent If feasible feature dictionaries, $\mathbf{X}$ and $\mathbf{Y}$, regarding row and column spaces are available, PCPF \citet{chiang16} makes use of these to generalize (\ref{eq:pcp}) to the below objective:
\begin{equation}
\label{eq:pcpf}
    \min_{\mathbf{H},\mathbf{S}}\ \ \Vert\mathbf{H}\Vert_*+\gamma\Vert\mathbf{S}\Vert_1\quad\text{subject to}\quad\mathbf{X}\mathbf{H}\mathbf{Y}^T+\mathbf{S}=\mathbf{M},
\end{equation}
for the same $\gamma$ as in (\ref{eq:pcp}). Convergence to the RPCA solution has only been established for the random sparsity model.\\
\indent AltProj \citet{netrapalli14} addresses RPCA by minimizing an entirely different objective:
\begin{equation}
\begin{split}
\label{eq:altproj}
    &\min_{\mathbf{L},\mathbf{S}}\ \ \Vert\mathbf{M} - \mathbf{L} -\mathbf{S}\Vert_F\\\text{subject to}\quad&\mathbf{L}\in\text{ set of low-rank matrices}\\&\mathbf{S}\in\text{ set of sparse matrices},
\end{split}
\end{equation}
where the search consists of alternating non-convex projections. That is, during each cycle, hard-thresholding takes place first to remove large entries and projection of appropriate residuals onto the set of low-rank matrices with increasing ranks is carried out next. Exact recovery has also been established.\\
\indent Fast RPCA \citet{yi16} follows yet another non-convex approach to solve RPCA. After an initialization stage, fast RPCA updates bilinear factors $\mathbf{U}$, $\mathbf{V}$ such that $\mathbf{L}=\mathbf{U}\mathbf{V}^T$ through a series of projected gradient descent and sparse estimations, where $\mathbf{U}$, $\mathbf{V}$ minimize the following loss:
\begin{equation}
	\min_{\mathbf{U},\mathbf{V}}\frac{1}{2}\Vert\mathbf{U}\mathbf{V}^T+\mathbf{S}-\mathbf{M}\Vert_F^2+\frac{1}{8}\Vert\mathbf{U}^T\mathbf{U}-\mathbf{V}^T\mathbf{V}\Vert_F^2,
\end{equation}
for $\mathbf{U}$, $\mathbf{V}$ properly constrained. Recovery guarantee is ensured.\\
\indent IRPCA-IHT \citet{niranjan17} includes features $\mathbf{X}$, $\mathbf{Y}$ in an iterative non-convex projection algorithm. Similar to AltProj, at each step, a new sparse estimate is calculated from hard thresholding via a monotonically decreasing threshold. After that, spectral hard thresholding takes place to attain the low-rank estimate. IRPCA-IHT provably converges to the solution of RPCA.\\
\indent We also mention here several works of non-convex objectives \citet{oh15,shang17}, though exact recovery guarantees are lacking.

\section{Problem Setup}
Suppose that there is a known data matrix $\mathbf{M}\in\mathbb{R}^{n_1\times n_2}$, which can be  decomposed into a low-rank component $\mathbf{L}^*$ and a sparse error  matrix $\mathbf{S}^*$ of compatible dimensions. Our aim is to identify these underlying matrices and hence robustly recover the low-rank component with the help of available side information in the form of feature matrices $\mathbf{X}$ and $\mathbf{Y}$.\\
\indent Concretely,  let  $\mathbf{L}^*=\mathbf{U}^*\mathbf{\Sigma}^*\mathbf{V}^{*T}$be the  singular value decomposition and $\mathbf{P}^*=\mathbf{X}^T\mathbf{U}^*\mathbf{\Sigma}^{*\frac{1}{2}}$ and $\mathbf{Q}^*=\mathbf{Y}^T\mathbf{V}^*\mathbf{\Sigma}^{*\frac{1}{2}}$. $\mathbf{S}^*$ follows the random sparsity model. That is, the support of $\mathbf{S}^*$ is chosen uniformly at random from the collection of all support sets of the same size. Furthermore, let us be informed of the proportion of non-zero entries per row and column, denoted by $\alpha$. Assume that there are also available features $\mathbf{X}\in\mathbb{R}^{n_1\times d_1}$ and $\mathbf{Y}\in\mathbb{R}^{n_2\times d_2}$ such that they are feasible, i.e. col($\mathbf{X}$)$\supseteq$col($\mathbf{U}^*$) and col($\mathbf{Y}$)$\supseteq$col($\mathbf{V}^*$) where col($\mathbf{A}$) is the column space of $\mathbf{A}$ and $\mathbf{X}^T\mathbf{X}=\mathbf{Y}^T\mathbf{Y}=\mathbf{I}$\footnote{This can always achieved via orthogonalisation.}. 

\indent In this paper, we discuss robust low-rank recovery using the above mentioned features and three different incoherence conditions: (i) $\Vert\mathbf{U}^*\Vert_{2,\infty}\le\sqrt{\frac{\mu_1r}{n_1}}$ and $\Vert\mathbf{V}^*\Vert_{2,\infty}\le\sqrt{\frac{\mu_1r}{n_2}}$; (ii) $\Vert\mathbf{X}\Vert_{2,\infty}\le\sqrt{\frac{\mu_2 d_1}{n_1}}$ and $\Vert\mathbf{Y}\Vert_{2,\infty}\le\sqrt{\frac{\mu_2 d_2}{n_2}}$; (iii) both (i) and (ii), where $r$ is the given rank of $\mathbf{L}^*$ and $\mu_1$, $\mu_2$ are constants.

\section{Algorithm}
	We use a non-convex approach to achieve the above objective. The algorithm consists of an initialization phase followed by a gradient descent phase. At each stage, we keep track of the factors $\mathbf{P}$, $\mathbf{Q}$ such that $\mathbf{L}=\mathbf{X}\mathbf{P}\mathbf{Q}^T\mathbf{Y}^T$.
	
	\subsection{Hard-thresholding}
	We first introduce the sparse estimator via hard-thresholding which is used in both phases. Given a threshold $\theta$, $\mathcal{T}_\theta(\mathbf{A})$ removes elements of $\mathbf{A}$ that are not among the largest $\theta$-fraction of elements in their respective rows and columns, breaking ties arbitrarily for equal elements:
	\begin{equation}
		\mathcal{T}_\theta(\mathbf{A})_{ij} =
		\begin{cases}
			0 & \quad \text{if } |\mathbf{A}_{ij}|\le \mathbf{A}^\theta{i\cdot} \text{ and } |\mathbf{A}_{ij}|\le \mathbf{A}^\theta{\cdot j},\\
			\mathbf{A}_{ij}  & \quad \text{otherwise},\\
		\end{cases}
	\end{equation}
	where $\mathbf{A}^\theta{i\cdot},\mathbf{A}^\theta{\cdot j}$ are the $(n_2\theta)^\text{th}$ and $(n_1\theta)^\text{th}$ largest element in absolute value in row $i$ and column $j$ respectively.
	
	\subsection{Initialization}
	$\mathbf{S}$ is first initialized as $\mathbf{S}_0=\mathcal{T}_\alpha(\mathbf{M})$. Next, we obtain $\mathbf{U}_0\mathbf{\Sigma}_0\mathbf{V}^T_0$ as the $r$-truncated SVD of $\mathbf{L}_0$, which is calculated via $\mathbf{L}_0 = \mathbf{M} - \mathbf{S}_0$. We can then construct $\mathbf{P}_0=\mathbf{X}^T\mathbf{U}_0\mathbf{\Sigma}_0^{\frac{1}{2}}$ and $\mathbf{Q}_0=\mathbf{Y}^T\mathbf{V}_0\mathbf{\Sigma}_0^{\frac{1}{2}}$. Such an initialization scheme gives $\mathbf{P}$, $\mathbf{Q}$ the desirable properties for use in the second phase.
	
	\subsection{Gradient Descent}
	In case (i), we need the following sets:
	\begin{equation}
		\mathcal{P}=\{\mathbf{A}\in\mathbb{R}^{d_1\times r}|\Vert\mathbf{X}\mathbf{A}\Vert_{2,\infty}\le\sqrt\frac{2\mu_1r}{n_1}\Vert\mathbf{P}_0\Vert_2\},
	\end{equation}
	\begin{equation}
		\mathcal{Q}=\{\mathbf{A}\in\mathbb{R}^{d_2\times r}|\Vert\mathbf{Y}\mathbf{A}\Vert_{2,\infty}\le\sqrt\frac{2\mu_1r}{n_2}\Vert\mathbf{Q}_0\Vert_2\}.
	\end{equation}
	Otherwise, we can simply take $\mathcal{P}$ as $\mathbb{R}^{d_1\times r}$ and $\mathcal{Q}$ as $\mathbb{R}^{d_2\times r}$.\\
	\indent To proceed, we first regularise $\mathbf{P}_0$ and $\mathbf{Q}_0$:
	\begin{equation}
		\mathbf{P}=\mathbf{\Pi}_\mathcal{P}(\mathbf{P}_0),\,\mathbf{Q}=\mathbf{\Pi}_\mathcal{Q}(\mathbf{Q}_0).
	\end{equation}
	\indent At each iteratiion, we first update $\mathbf{S}$ with the sparse estimator using a threshold of $\alpha+\min(10\alpha+0.1)$:
	\begin{equation}
		\mathbf{S} = \mathcal{T}_{\alpha+\min(10\alpha+0.1)}(\mathbf{M} - \mathbf{X}\mathbf{P}\mathbf{Q}^T\mathbf{Y}^T).
	\end{equation}
	\indent For $\mathbf{P}$, $\mathbf{Q}$, we define the following objective function 
	\begin{equation}
		\mathcal{L}(\mathbf{P},\mathbf{Q})=\frac{1}{2}\Vert\mathbf{X}\mathbf{P}\mathbf{Q}^T\mathbf{Y}^T+\mathbf{S}-\mathbf{M}\Vert_F^2+\frac{1}{64}\Vert\mathbf{P}^T\mathbf{P}-\mathbf{Q}^T\mathbf{Q}\Vert_F^2.
	\end{equation}
	$\mathbf{P}$ and $\mathbf{Q}$ are updated by minimizing the above function subject to the constraints imposed by the sets $\mathcal{P}$ and $\mathcal{Q}$. That is,
	\begin{equation}
		\mathbf{P} = \mathbf{\Pi}_\mathcal{P}(\mathbf{P}-\eta	\nabla_\mathbf{P}\mathcal{L}),
	\end{equation}
	\begin{equation}
		\mathbf{Q} = \mathbf{\Pi}_\mathcal{Q}(\mathbf{Q}-\eta	\nabla_\mathbf{Q}\mathcal{L}),
	\end{equation}
	where the step size $\eta$ is determined analytically below. With properly initialized $\mathbf{P}$ and $\mathbf{Q}$, such an optimization design converges to $\mathbf{P}^*$ and $\mathbf{Q}^*$.
	The  procedure is summarized in Algorithm \ref{alg:algo}.
	\begin{algorithm}[tb]
		\caption{Non-convex solver for robust principal component analysis with features}
		\label{alg:algo}
		\begin{algorithmic}[1]
			\REQUIRE Observation $\mathbf{M}$, features $\mathbf{X},\mathbf{Y}$, rank $r$, corruption approximation $\alpha$ and step size $\eta$.
			\STATEx {\bfseries Initialization:} \STATE $\mathbf{S}=\mathcal{T}_\alpha(\mathbf{M})$
			\STATE $\mathbf{U}\mathbf{\Sigma}\mathbf{V}^T=r$-SVD$(\mathbf{M} - \mathbf{S})$
			\STATE $\mathbf{P}=\mathbf{X}^T\mathbf{U}\mathbf{\Sigma}^{\frac{1}{2}}$
			\STATE $\mathbf{Q}=\mathbf{Y}^T\mathbf{V}\mathbf{\Sigma}^{\frac{1}{2}}$
			\STATEx {\bfseries Gradient descent:}
			\STATE$\mathbf{P}=\mathbf{\Pi}_\mathcal{P}(\mathbf{P})$
			\STATE$\mathbf{Q}=\mathbf{\Pi}_\mathcal{Q}(\mathbf{Q})$
			\WHILE{not converged}
			\STATE $\mathbf{S} = \mathcal{T}_{\alpha+\min(10\alpha+0.1)}(\mathbf{M} - \mathbf{X}\mathbf{P}\mathbf{Q}^T\mathbf{Y}^T)$
			\STATE $\mathbf{P} =\mathbf{\Pi}_\mathcal{P}(\mathbf{P}-\eta	\nabla_\mathbf{P}\mathcal{L})$
			\STATE $\mathbf{Q} = \mathbf{\Pi}_\mathcal{Q}(\mathbf{Q}-\eta	\nabla_\mathbf{Q}\mathcal{L})$
			\ENDWHILE
			\ENSURE $\mathbf{L}=\mathbf{X}\mathbf{P}\mathbf{Q}^T\mathbf{Y}^T$, $\mathbf{S}$
		\end{algorithmic}
	\end{algorithm}

\section{Analysis}
	We first provide theoretical justification of our proposed approach. Then we evaluate its computational complexity. The proofs can be found in the supplementary material.
	
	\subsection{Convergence}
	The initialization phase provides us with the following guarantees on $\mathbf{P}$ and $\mathbf{Q}$.
	\begin{innercustomthm}
		In cases (i) and (iii), if $\alpha\le\frac{1}{16\kappa r\mu_1}$, we have
		\begin{equation}
			d(\mathbf{P}_0,\mathbf{Q}_0,\mathbf{P}^*,\mathbf{Q}^*)\le18\alpha r\mu_1\sqrt{r\kappa\sigma_1^*}.
		\end{equation}
		In case (ii), if $\alpha\le\frac{1}{16\kappa\mu_2\sqrt{d_1d_2}}$, we have
		\begin{equation}
			d(\mathbf{P}_0,\mathbf{Q}_0,\mathbf{P}^*,\mathbf{Q}^*)\le18\alpha\mu_2\sqrt{rd_1d_2\kappa\sigma_1^*},
		\end{equation}
		where $\kappa$ is the condition number of $L^*$ and $d$ is a distance metric defined in the appendix.
	\end{innercustomthm}
	\begin{innercustomthm}
		For $\eta\le\frac{1}{192\Vert\mathbf{L}_0\Vert_2}$, there exist constants $c_1>0$, $c_2>0$, $c_3>0$, $c_4>0$, $c_5>0$ and $c_6>0$ such that, in case (i), when $\alpha\le\frac{c_1}{\mu_1(\kappa r)^\frac{3}{2}}$, we have the following relationship
		\begin{equation}
		d(\mathbf{P}_t,\mathbf{Q}_t,\mathbf{P}^*,\mathbf{Q}^*)^2\le(1-c_2\eta\sigma_r^*)^td(\mathbf{P}_0,\mathbf{Q}_0,\mathbf{P}^*,\mathbf{Q}^*)^2,
		\end{equation}
		in case (ii), when $\alpha\le\frac{c_3}{\mu_2dr^\frac{1}{2}\kappa^\frac{3}{2}}$, we have
		\begin{equation}
		d(\mathbf{P}_t,\mathbf{Q}_t,\mathbf{P}^*,\mathbf{Q}^*)^2\le(1-c_4\eta\sigma_r^*)^td(\mathbf{P}_0,\mathbf{Q}_0,\mathbf{P}^*,\mathbf{Q}^*)^2.
		\end{equation}
		and in case (iii), when
		$\alpha\le c_5\min(\frac{1}{\mu_2d\kappa},\frac{1}{\mu_1(\kappa r)^\frac{3}{2}})$, we have
		\begin{equation}
		d(\mathbf{P}_t,\mathbf{Q}_t,\mathbf{P}^*,\mathbf{Q}^*)^2\le(1-c_6\eta\sigma_r^*)^td(\mathbf{P}_0,\mathbf{Q}_0,\mathbf{P}^*,\mathbf{Q}^*)^2.
		\end{equation}
	\end{innercustomthm}
	
	\subsection{Complexity}
	From \textbf{Theorem 2}, it follows that our algorithm converges at a linear rate under assumptions (ii) and (iii). To converge below $\epsilon$ of the initial error, $O(\text{log}(\frac{1}{\epsilon}))$ iterations are needed. At each iteration, the most costly step is matrix multiplication which takes $O(rn^2)$ time. Overall, our algorithm has total running time of $O(rn^2\text{log}(\frac{1}{\epsilon}))$.
	
\section{Experimental results}
	\begin{figure}[h!]
	\centering
\includegraphics[width=0.8\linewidth]{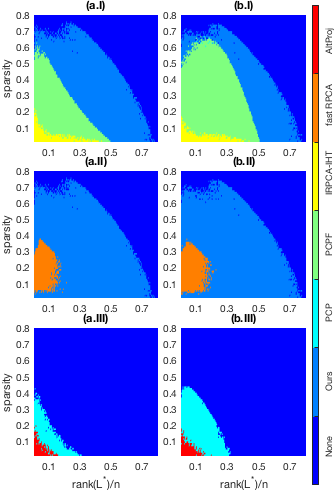}
   \caption{Domains of recovery by various algorithms: (a) for random signs and
(b) for coherent signs.}
    \label{fig:recover}
\end{figure}
	We have found that when the step size is set to 0.5, reasonable results can be obtained. For all algorithms in comparison, we run a total of 3000 iterations or until $\Vert \mathbf{M} - \mathbf{L} - \mathbf{S}\Vert_F/\Vert\mathbf{M}\Vert_F<10^{-7}$ is met.
	\subsection{Phase transition}
    Here, we vary the rank and the error sparsity to investigate the behavior of both our algorithm and existing state-of-art algorithms in terms of recoverability. True low-rank matrices are created via $\mathbf{L}^*=\mathbf{J}\mathbf{K}^T$, where $200\times r$ matrices $\mathbf{J},\mathbf{K}$ have independent elements drawn randomly from a Gaussian distribution of mean $0$ and variance $5\cdot10^{-3}$ so $r$ becomes the rank of $\mathbf{L}^*$. Next, we corrupt each column of $\mathbf{L}^*$ such that $\alpha$ of the elements are set independently with magnitude $\mathcal{U}(0,\frac{r}{40})$. However, this does not guarantee $\alpha$ row corruption. We thus select only matrices whose maximum row corruption does not exceed $\alpha+6.5\%$ but we still feed $\alpha$ to the algorithms in order to demonstrate that our algorithm does not need the exact value of corruption ratio. We consider two types of signs for error: Bernoulli $\pm1$ and $\text{sgn}(\mathbf{L}^*)$. The resulting $\mathbf{M}$ thus becomes the simulated observation. In addition, let $\mathbf{L}^*=\mathbf{U}\mathbf{\Sigma}\mathbf{V}^T$ be the SVD of $\mathbf{L}^*$. Feature $\mathbf{X}$ is formed by randomly interweaving column vectors of $\mathbf{U}$ with 5 arbitrary orthonormal bases for the null space of $\mathbf{U}^T$, while permuting the expanded columns of $\mathbf{V}$ with 5 random orthonormal bases for the kernel of $\mathbf{V}^T$ forms feature $\mathbf{Y}$. Hence, the feasibility conditions are fulfilled: col$(\mathbf{X})\supseteq$col$(\mathbf{L}_0)$, col$(\mathbf{Y})\supseteq$col$(\mathbf{L}_0^T)$. For each $(r,\alpha)$ pair, three observations are constructed. The recovery is successful if for all these three problems, 
\begin{equation}
    \frac{\Vert\mathbf{L}-\mathbf{L}^*\Vert_F}{\Vert\mathbf{L}^*\Vert_F}<10^{-3}
\end{equation}
from the recovered $\mathbf{L}$.\\

\indent Figures \ref{fig:recover}(I) plot results from algorithms incorporating features.
Besides, our algorithm contrasts with fast RPCA in Figure \ref{fig:recover}(II). Other feature-free algorithms are investigated in Figure \ref{fig:recover}(III). Figures \ref{fig:recover}(a) illustrate the random sign model and Figures \ref{fig:recover}(b) for the coherent sign model. All previous non-convex attempts fail to outperform their convex equivalents. IRPCA-IHT is unable to deal with even moderate levels of corruption. The frontier of recoverability that has been advanced by our algorithm over PCPF is phenomenal, massively ameliorating fast RPCA. The anomalous asymmetry in the two sign models is no longer observed in non-convex algorithms.

\subsection{Running Time}
\begin{figure}[h!]
\centering
\includegraphics[width=0.8\linewidth]{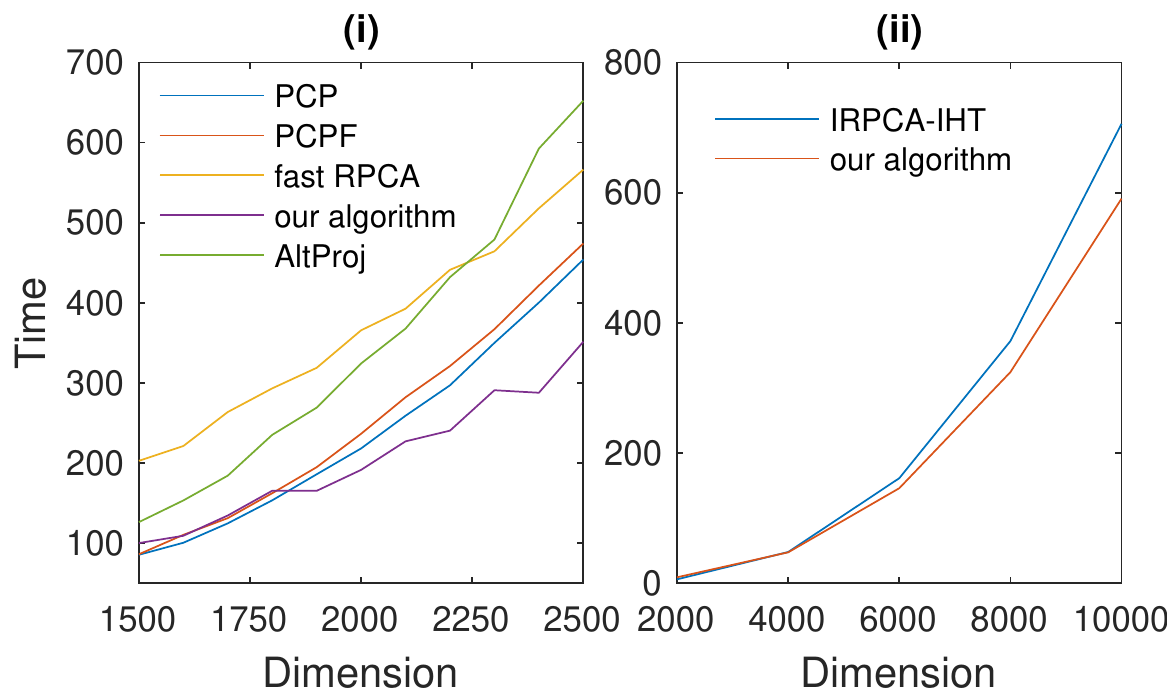}
   \caption{(i) Running times for observation matrices of increasing dimensions for (i) PCP, PCPF, fast RPCA, AltProj, our algorithm and (ii) IRPCA-IHT and our algorithm when $\frac{\Vert\mathbf{L}-\mathbf{L}^*\Vert_F}{\Vert\mathbf{L}^*\Vert_F}\le1\%$.}
    \label{fig:time}
\end{figure}
Next, we highlight the speed of our algorithm for large-scale matrices, typical of video sequences \citet{xiong16}. 1500$\times$1500 to 2500$\times$2500 random observation matrices are generated, where the rank is chosen to be 20$\%$ of the column number and random sign error corrupts 11$\%$ of the entries, with features $\mathbf{X},\mathbf{Y}$ having a dimension of 50$\%$ of the column number. The running times of all algorithms except IRPCA-IHT are plotted in \ref{fig:time} (i) because IRPCA-IHT is not able to achieve a relative error ($\frac{\Vert\mathbf{L}-\mathbf{L}^*\Vert_F}{\Vert\mathbf{L}^*\Vert_F}$) less than 1$\%$ for larger matrices. For fair comparison, we have relaxed the rank to 0.3$\%$ of the column number and error rate to 0.1$\%$ to compare our algorithm with IRPCA-IHT for matrices ranging from 2000$\times$2000 to 10000$\times$10000. We have used features $\mathbf{X},\mathbf{Y}$ having a dimension of 80$\%$ of the column number to speed up the process. The result is shown in Figure \ref{fig:time} (ii). All times are averaged over three trials. It is evident that, for large matrices, our algorithm overtakes all existing algorithms in terms of speed. Note that features in PCPF even slow down the recovery process.

\begin{table*}
\centering
\begin{tabular}{ |c|c|c|c|c|c|c|c|c| }
\hline
$\alpha$ & clean & noisy & PCP & PCPF & AltProj & IRPCA-IHT & fast RPCA & our algorithm \\
\hline
10 & & 30.45 & 82.75 & 83.35 & 81.4 & 65.2 & 81.1 & \textbf{86.9} \\
\cline{0-0}\cline{3-9}
15 & & 25.1 & 82.95 & 83.4 & 81.15 & 49.65 & 79.65 & \textbf{84.8} \\
\cline{0-0}\cline{3-9}
20 & 89.65 & 23.15 & 83.5 & 84 & 79.3 & 37.8 & 78.65 & \textbf{83.8} \\
\cline{0-0}\cline{3-9}
25 & & 18.65 & 81.35 & 82.65 & 74.05 & 30.35 & 75.3 & \textbf{83.15} \\
\cline{0-0}\cline{3-9}
30 & & 18.6 & 77.95 & 79 & 71.5 & 24.1 & 72.9 & \textbf{82.05} \\
\cline{0-0}\cline{3-9}
35 & & 16.95 & 71.2 & 73.4 & 67.75 & 21.05 & 71.45 & \textbf{79.05} \\
\hline
\end{tabular}
\caption{Classification results obtained by a linear SVM.}
\label{table:linear}
\end{table*}

\begin{table*}
\centering
\begin{tabular}{ |c|c|c|c|c|c|c|c|c| }
\hline
$\alpha$ & clean & noisy & PCP & PCPF & AltProj & IRPCA-IHT & fast RPCA & our algorithm \\
\hline
10 & & 87 & 87.25 & 87.3 & 86.45 & 89.3 & 89.25 & \textbf{90.3} \\
\cline{0-0}\cline{3-9}
15 & & 75.85 & 87.15 & 87.4 & 86.75 & 82.85 & 87.2 & \textbf{89.8} \\
\cline{0-0}\cline{3-9}
20 & 92.25 & 64.35 & 87.6 & 87.55 & 84.65 & 71.2 & 85.55 & \textbf{88.55} \\
\cline{0-0}\cline{3-9}
25 & & 55.85 & 87 & 86.95 & 79.4 & 62.35 & 82.65 & \textbf{87.8} \\
\cline{0-0}\cline{3-9}
30 & & 47.15 & 81.15 & 81.55 & 76.75 & 53.5 & 78.3 & \textbf{85.65} \\
\cline{0-0}\cline{3-9}
35 & & 40.55 & 74.8 & 75.7 & 71 & 47.4 & 76.75 & \textbf{85.15} \\
\hline
\end{tabular}
\caption{Classification results obtained by an SVM with RBF kernel.}
\label{table:kernel}
\end{table*}

\subsection{Image Classification}
\begin{figure}[t!]
\begin{center}
\includegraphics[width=0.6\linewidth]{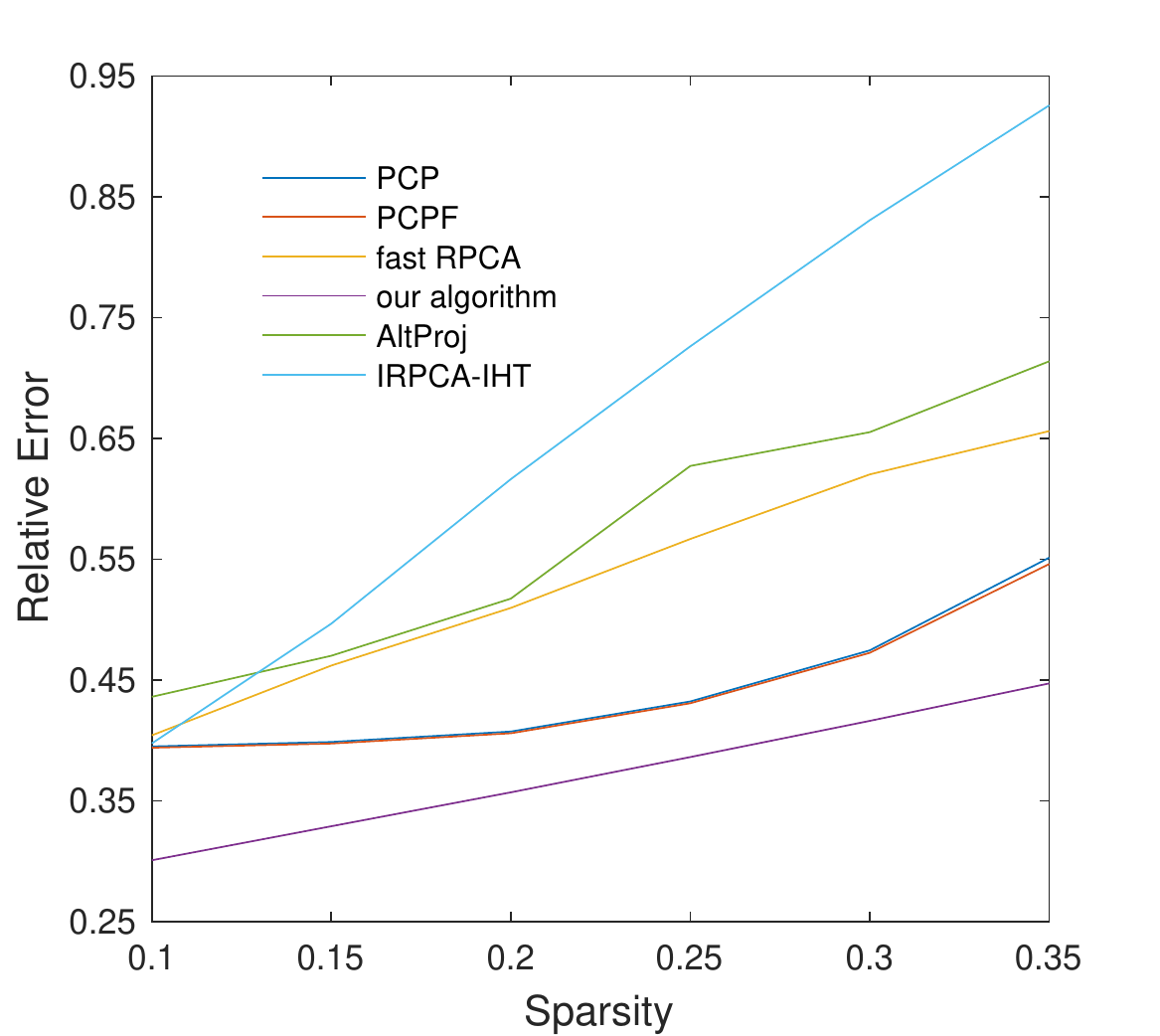}
\end{center}
   \caption{Relative error ($\frac{\Vert\mathbf{L}-\mathbf{L}^*\Vert_F}{\Vert\mathbf{L^*}\Vert_F}$) for sparsity values: 10$\%$, 15$\%$, 20$\%$, 25$\%, 
   30$\%, 35$\%$.}
    \label{fig:mnistError}
\end{figure}
Once images are denoised, classification can be performed on them. The classification results directly reflect the image denoising ability. For a set of correlated images, low-rank algorithms are normally used to remove noise that is sparse. The same classifier is thus able to compare the different low-rank models.

The MNIST dataset is such an example which contains hand-written digits divided into training and testing sets. Let the observation matrix be composed of 2000 vectorized random images from the test set stacked column-wise. In this case, the left feature obtained from the training set is also applicable to the test set because of the Eigendigit nature. This imparts our algorithm to supervised learning where there are clean related training samples available. The right feature does not posses such property and is set to the identity matrix. We add a range of sparse noise to the test set separately where the noise sets the pixel to 255. For PCPF, we take $d=300$ as in \citet{chiang16} and for IRPCA-IHT and our algorithm we use $d=150$ instead. 

The relative error between the recovered matrix by the competing algorithms and the clean test matrix is plotted in Figure \ref{fig:mnistError}. Our algorithm is most accurate in removing the added artificial noise. To evaluate how classifiers perform on the recovered matrices, we train the linear and kernel SVM using the training set and test the corresponding models on the recovered images. Table \ref{table:linear} tabulates the linear SVM. Table \ref{table:kernel} tabulates the kernel SVM. Both classifiers confirm the recovery result obtained by various models corroborating our algorithm's  pre-eminent accuracy.

\subsection{Face denoising}
\begin{figure}[h!]
\centering
\includegraphics[width=0.8\linewidth]{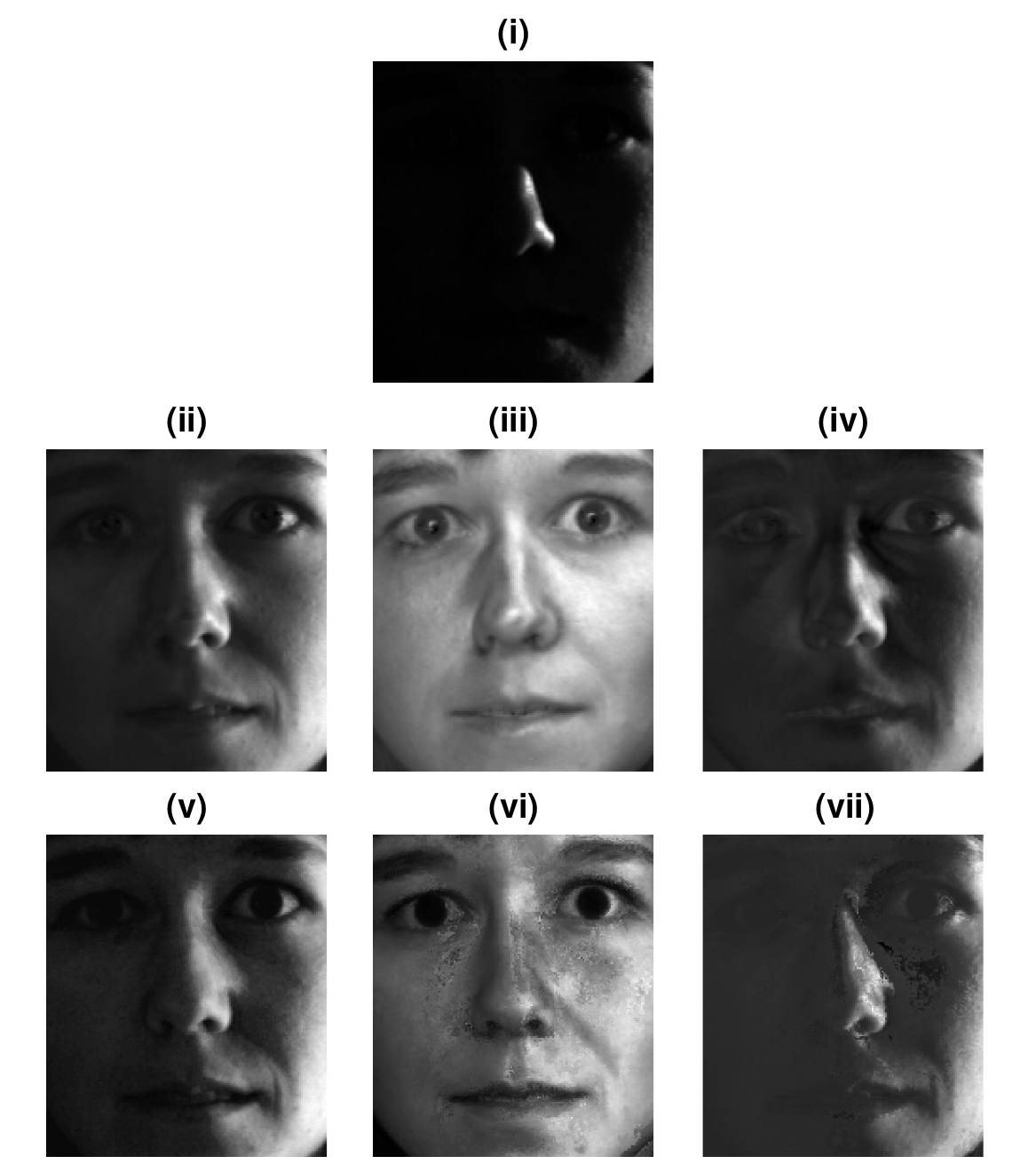}
   \caption{(i) original; (ii) PCPF; (iii) our algorithm; (iv) IRPCA-IHT; (v) PCP; (vi) fast RPCA; (vii) AltProj.}
    \label{fig:yale}
\end{figure}
It is common practice to decompose raw facial images as a low-rank component for faithful face representation and a sparse component for defects. This is because the face is a convex Lambertian surface which under distant and isotropic lighting has an underlying model that spans a 9-D linear subspace \citet{basri03}, but theoretical lighting conditions cannot be realised and there are unavoidable occlusion and albedo variations in real images. We demonstrate that there can be a substantial boost to the performance of facial denoising by leveraging dictionaries learnt from the images themselves.\\
\indent The extended Yale B database is used as our observation which consists images under different illuminations for a fixed pose. We study all 64 images of a randomly chosen person. A $32556\times64$ observation matrix is formed by vectorizing each $168\times192$ image. For fast RPCA and our algorithm, a sparsity of 0.2 is adopted. We learn the feature dictionary as in \citet{xue17}. In a nutshell, the feature learning process can be treated as a sparse encoding problem. More specifically, we simultaneously seek a dictionary $\mathbf{D}\in\mathbb{R}^{n_1\times c}$ and a sparse representation $\mathbf{B}\in\mathbb{R}^{c\times n_2}$ such that:
\begin{equation}
    \begin{split}
        &\minimise_{\mathbf{D},\mathbf{B}}\quad\Vert\mathbf{M}-\mathbf{D}\mathbf{B}\Vert_F^2\\
        &\subject\text{ to}\ \ \,\gamma_i\le t\text{  for }i=1\dots n_2,\\
    \end{split}
\end{equation}
where $c$ is the number of atoms, $\gamma_i$'s count the number of non-zero elements in each sparsity code and $t$ is the sparsity constraint factor. This can be solved by the K-SVD algorithm. Here, feature $\mathbf{X}$ is the dictionary $\mathbf{D}$, feature $\mathbf{Y}$ corresponds to a similar solution using the transpose of the observation matrix as input. We set $c$ to $40$, $t$ to $40$ and used $10$ iterations.

As a visual illustration, recovered images from all algorithms are exhibited in Figure \ref{fig:yale}. For this challenging scenario, our algorithm totally removed all shadows. PCPF is smoother than PCP but still suffers from shade. AltProj and fast RPCA both introduced extra artefacts. Although IRPCA-IHT managed to remove the shadows but brought back a severely distorted image. To quantitatively verify the improvement made by our proposed method, we examine the structural information contained within the denoised eigenfaces. Singular values of the recovered low-rank matrices from all algorithms are plotted in Figure \ref{fig:rank}. All non-convex algorithms are competent in incorporating the rank information to keep only 9 singular values, vastly outperforming convex approaches. Among them, our algorithm has the most rapid decay that is found naturally \citet{wright11}.

\begin{figure}[h!]
\centering
\includegraphics[width=0.6\linewidth]{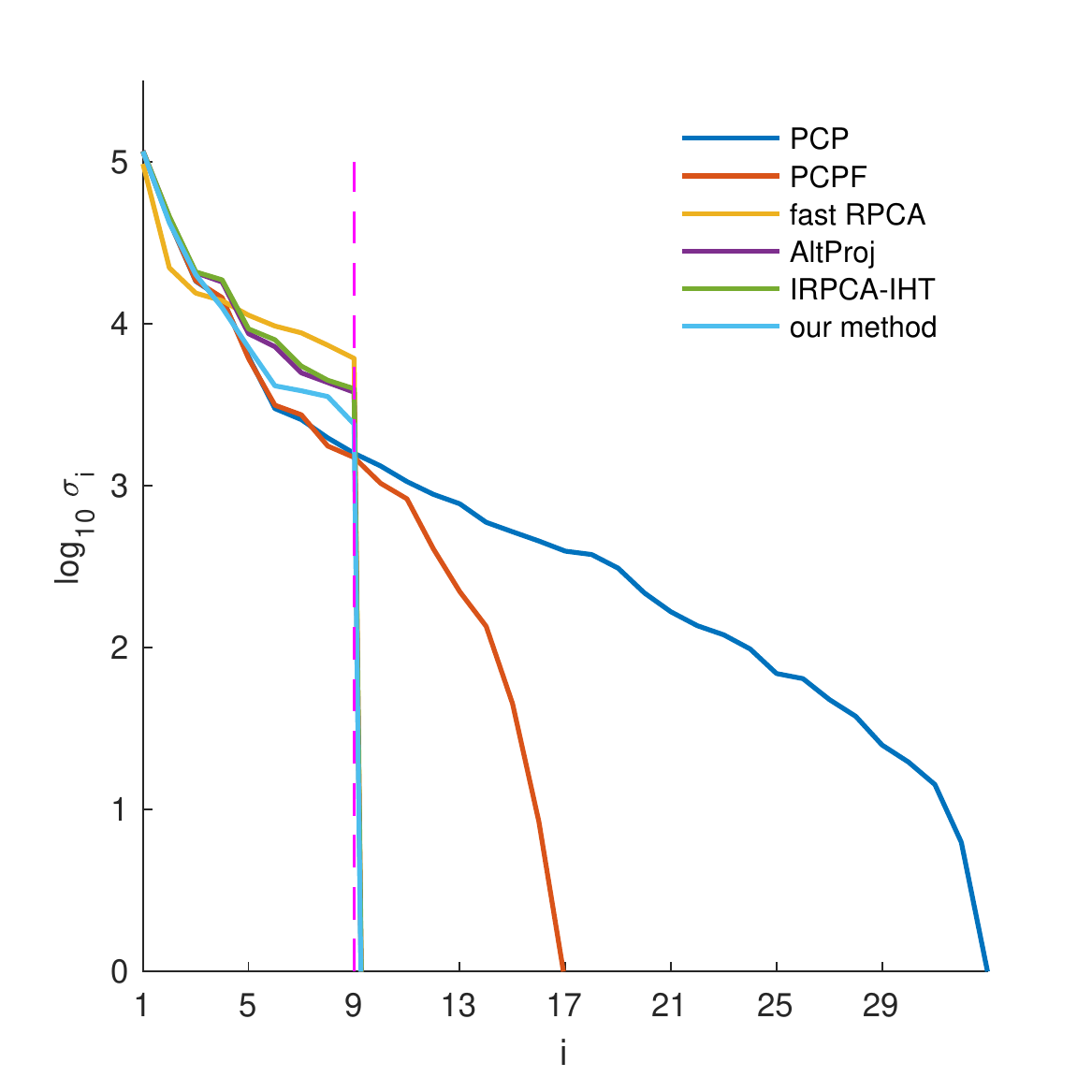}
   \caption{Log-scale singular values of the denoised matrices.}
    \label{fig:rank}
\end{figure}

\section{Conclusion}
This work proposes a new non-convex algorithm to solve RPCA with the help of features when the error sparsity is roughly known. Exact recovery guarantee has been established for three different assumptions about the incoherence conditions on features and the data observation matrix. Simulation experiments suggest that our algorithm is able to recover matrices of higher ranks corrupted by errors of higher sparsity than previous state-of-the-art approaches. Large synthetic matrices also show that our algorithm scales best with observation matrix dimension. MNIST and Yale B datasets further justify that our algorithm leads other approaches by a fair margin. Future work may involve finding a more accurate initialization scheme.	

\newpage

\begin{appendices}

\section{Convex Projection}
Given $\mathbf{P}$, the problem of finding $\mathbf{\Pi}_\mathcal{P}(\mathbf{P})$ can be seen as projection onto the intersection of a series of closed convex sets $\mathcal{P}_i$, that is $\mathcal{P}=\mathcal{P}_1\bigcap\cdots\bigcap\mathcal{P}_{d_1}$, where $\mathcal{P}_i=\{\mathbf{A}\in\mathbb{R}^{d_1\times r}||\mathbf{X}_{i\cdot}\mathbf{A}|_2\le\sqrt\frac{2\mu_1r}{n_1}\Vert\mathbf{P}_0\Vert_2\}$. We have emperically found that the Cyclic Dykstra algorithm \citet{reich12} has the fastest rate of convergence. Let $\mathbf{A}_0=\mathbf{P}$, and $\mathbf{B}_{-(d_1-1)}=\mathbf{B}_{-(d_1-2)}=\cdots=\mathbf{B}_{-1}=\mathbf{B}_0=\mathbf{0}\in\mathbb{R}^{d_1\times r}$, the Cyclic Dykstra algorithm updates, at each iteration, $\mathbf{A}_{k+1}=\mathbf{\Pi}_{\mathcal{P}_{k+1\bmod d_1}}(\mathbf{A}_k+\mathbf{B}_{k+1-d_1})$ and $\mathbf{B}_{k+1}=\mathbf{A}_k+\mathbf{B}_{k+1-d_1}-\mathbf{A}_{k+1}$.\\\\
For $\mathbf{\Pi}_{\mathcal{P}_i}(\mathbf{P})$, we formulate the equivalent optimisation problem below
\begin{equation}
    \min_\mathbf{A}\Vert\mathbf{A}-\mathbf{P}\Vert_F^2\quad\text{s.t.}\quad|\mathbf{X}_{i\cdot}\mathbf{A}|_2=\sqrt\frac{2\mu_1r}{n_1}\Vert\mathbf{P}_0\Vert_2,
\end{equation}
for $|\mathbf{X}_{i\cdot}\mathbf{P}|_2>\sqrt\frac{2\mu_1r}{n_1}\Vert\mathbf{P}_0\Vert_2$. Its solution is given by
\begin{equation}
    A = (\mathbf{I}_{d_1\times d_1}+\frac{(\frac{|\mathbf{X}_{i\cdot}\mathbf{P}|_2}{\sqrt\frac{2\mu_1r}{n_1}\Vert\mathbf{P}_0\Vert_2}-1)\mathbf{X}_{i\cdot}^T\mathbf{X}_{i\cdot}}{|\mathbf{X}_{i\cdot}|_2^2})^{-1}\mathbf{P}.
\end{equation}
For $\mathbf{Q}$, $\mathbf{\Pi}_\mathcal{Q}(\mathbf{Q})$ follows similarly.\\\\
We have also run experiments to see how much improvement can be gained by convex projection. 200$\times$200 high-incoherence matrices are created with ranks from 140 to 155 and corrupted by 10$\%$ random sign errors. Our algorithm is applied with projection several times. Each uses a different number of iterative steps ranging from 0 to 2000. Recoverability is plotted against the number of iterative projections in Figure \ref{fig:projectionFeatureSubspace}. There is hardly any noticeable improvement so we do not use convex projection in our comparison experiments. Further analysis is demanded to justify the redundency of convex projection.

\begin{figure}[h!]
    \vspace*{-0.3cm}
\begin{center}
\includegraphics[width=0.6\linewidth]{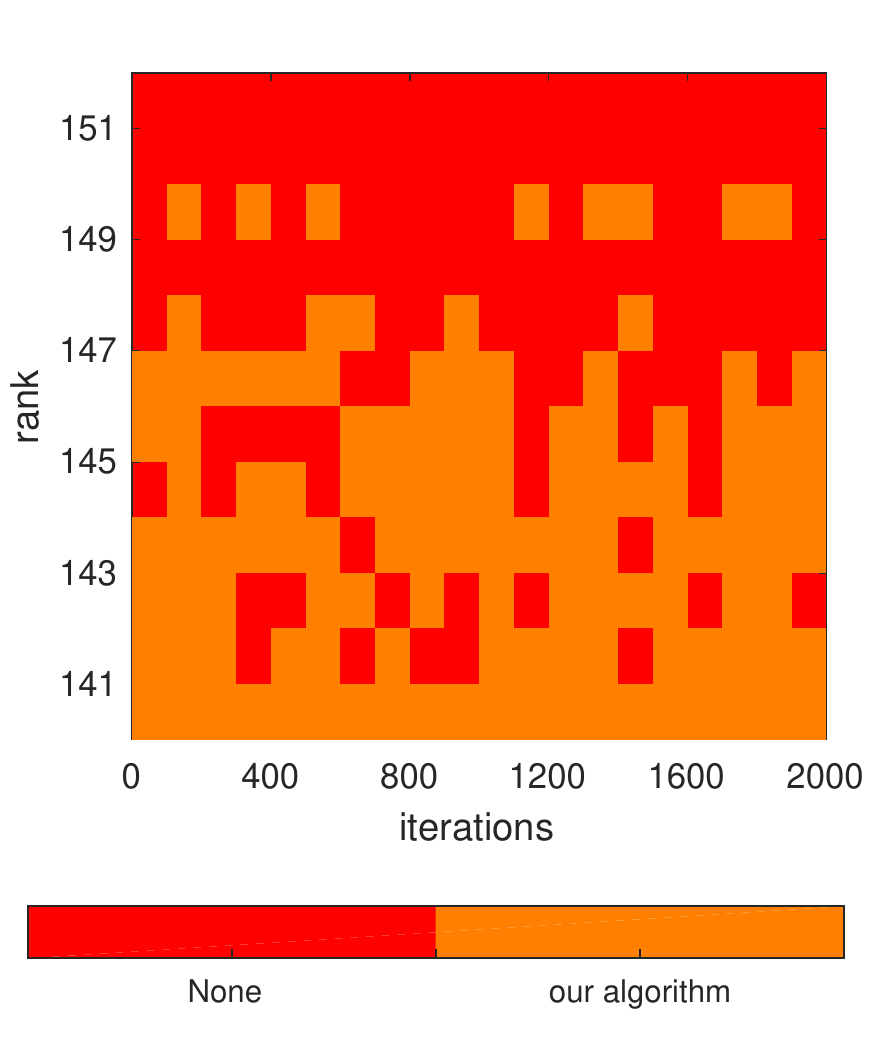}
\end{center}
    \vspace*{-0.6cm}
   \caption{Effectiveness of convex projection.}
    \label{fig:projectionFeatureSubspace}
    \vspace*{-0.3cm}
\end{figure}

\section{Proofs}
For simplicity, we assume that $n_1=n_2=n$, $d_1=d_2=d$.
\subsection{Proof of Theorem 1}
We first declare some lemmas that will be essential to our result.\\
\begin{lemma}\label{t11}
	Let $\mathbf{S}_0$ be obtained from the initialisation phase, we have
	\begin{equation}
		\Vert\mathbf{M}-\mathbf{S}_0-\mathbf{L}^*\Vert_\infty\le2\Vert \mathbf{L}^*\Vert_\infty.
	\end{equation}
\end{lemma}
\begin{proof}
	See \citet{yi16} theorem 1.
\end{proof}
\begin{lemma}\label{t12}
	For any matrix $\mathbf{A}\in\mathbb{R}^{n\times n}$ for which the proportion of non-zero entries per row and column is $\beta$, we have
	\begin{equation}
		\Vert\mathbf{A}\Vert_2\le\beta n\Vert\mathbf{A}\Vert_\infty.
	\end{equation}
\end{lemma}
\begin{proof}
	See \citet{netrapalli14} lemma 4.
\end{proof}
\begin{lemma}\label{t13}
	For two rank $r$ matrices $\mathbf{L}_1$ and $\mathbf{L}_2$ of the same dimension whose compact SVDs are $\mathbf{L}_1=\mathbf{U}_1\mathbf{\Sigma}_1\mathbf{V}_1^T$ and $\mathbf{L}_2=\mathbf{U}_2\mathbf{\Sigma}_2\mathbf{V}_2^T$, we have
	\begin{equation}
		d(\mathbf{U}_1\mathbf{\Sigma}_1^{\frac{1}{2}},\mathbf{V}_1\mathbf{\Sigma}_1^{\frac{1}{2}},\mathbf{U}_2\mathbf{\Sigma}_2^{\frac{1}{2}},\mathbf{V}_2\mathbf{\Sigma}_2^{\frac{1}{2}})^2\le\frac{2}{\sqrt{2}-1}\frac{\Vert \mathbf{L}_1-\mathbf{L}_2\Vert_F^2}{\sigma_r(\mathbf{L}_2)},
	\end{equation}
	provided $\Vert\mathbf{L}_1-\mathbf{L}_2\Vert_2\le\frac{1}{2}\sigma_r(\mathbf{L}_2)$.
\end{lemma}
\begin{proof}
	See \citet{tu16} lemma 5.14.
\end{proof}
\begin{lemma}\label{t14}
	For any matrices $\mathbf{A}$ and $\mathbf{B}$ of consistent sizes, we have
	\begin{equation}
		\Vert\mathbf{A}\mathbf{B}\Vert_{2,\infty}\le\Vert\mathbf{A}\Vert_{2,\infty}\Vert\mathbf{B}\Vert_2.
	\end{equation}
\end{lemma}
\begin{proof}
	See \citet{liu17} lemma 4.2.
\end{proof}
\begin{lemma}\label{t15}
	For any matrix $\mathbf{A}$ with compact SVD $\mathbf{A}=\mathbf{U}\mathbf{\Sigma}\mathbf{V}^T$,
	\begin{equation}
		\Vert\mathbf{A}\Vert_\infty\le\Vert\mathbf{\Sigma}\Vert_2\Vert\mathbf{U}\Vert_{2,\infty}\Vert\mathbf{V}\Vert_{2,\infty}.
	\end{equation}
\end{lemma}
\begin{proof}
	See \citet{yi16} theorem 1.
\end{proof}
\begin{lemma}\label{t16}
	Let $\mathbf{U}_0,\mathbf{V}_0, \mathbf{\Sigma}_0, \mathbf{S}_0$ be obtained from the initialisation phase, we have
	\begin{equation}
		\Vert\mathbf{U}_0\mathbf{\Sigma}_0\mathbf{V}_0^T-\mathbf{M}+\mathbf{S}_0\Vert_2\le\Vert\mathbf{M}-\mathbf{S}_0-\mathbf{L}^*\Vert_2
	\end{equation}
\end{lemma}
\begin{proof}
	Weyl's theorem tells us that, for $1\le i\le n$, $|\sigma_i(\mathbf{L}^*)-\sigma_i(\mathbf{M}-\mathbf{S}_0)|\le\Vert\mathbf{M}-\mathbf{S}_0-\mathbf{L}^*\Vert_2$. When $i=r+1$, $\sigma_i(\mathbf{L}^*)=0$ and $\sigma_i(\mathbf{M}-\mathbf{S}_0)=\Vert\mathbf{U}_0\mathbf{\Sigma}_0\mathbf{V}_0^T-\mathbf{M}+\mathbf{S}_0\Vert_2$ because $\mathbf{L}^*$ has rank $r$ and $\mathbf{U}_0\mathbf{\Sigma}_0\mathbf{V}_0^T=r$-SVD$(\mathbf{M}-\mathbf{S}_0)$.
\end{proof}
\begin{lemma}\label{t17}
	For $\mathbf{A},\mathbf{B},\mathbf{C},\mathbf{D}\in\mathbb{R}^{d\times r}$
	\begin{equation}
		d(\mathbf{X}^T\mathbf{A},\mathbf{Y}^T\mathbf{B},\mathbf{X}^T\mathbf{C},\mathbf{Y}^T\mathbf{D})\le d(\mathbf{A},\mathbf{B},\mathbf{C},\mathbf{D}).
	\end{equation}
\end{lemma}
\begin{proof}
	\begin{equation}
		\begin{split}
			&\quad\ d(\mathbf{X}^T\mathbf{A},\mathbf{Y}^T\mathbf{B},\mathbf{X}^T\mathbf{C},\mathbf{Y}^T\mathbf{D})\\&=\min_{\mathbf{R}}\sqrt{\Vert\mathbf{X}^T(\mathbf{A}-\mathbf{C}\mathbf{R})\Vert_F^2+\Vert\mathbf{Y}^T(\mathbf{B}-\mathbf{D}\mathbf{R})\Vert_F^2}\\&\le\min_{\mathbf{R}}\sqrt{\Vert\mathbf{X}^T\Vert_2^2\Vert(\mathbf{A}-\mathbf{C}\mathbf{R})\Vert_F^2+\Vert\mathbf{Y}^T\Vert_2^2\Vert(\mathbf{B}-\mathbf{D}\mathbf{R})\Vert_F^2}\\&=\min_{\mathbf{R}}\sqrt{\Vert(\mathbf{A}-\mathbf{C}\mathbf{R})\Vert_F^2+\Vert(\mathbf{B}-\mathbf{D}\mathbf{R})\Vert_F^2}\\&=d(\mathbf{A},\mathbf{B},\mathbf{C},\mathbf{D}).
		\end{split}
	\end{equation}
\end{proof}
\noindent\\We begin by deriving a bound on $\Vert \mathbf{M}-\mathbf{S}_0-\mathbf{L}^*\Vert_2$,
\begin{equation}
	\label{eq:result1}
	\begin{split}
		\Vert \mathbf{M}-\mathbf{S}_0-\mathbf{L}^*\Vert_2&\le2\alpha n\Vert \mathbf{M}-\mathbf{S}_0-\mathbf{L}^*\Vert_\infty\le4\alpha n\Vert\mathbf{L}^*\Vert_\infty\\&\le4\alpha n\Vert\mathbf{\Sigma}^*\Vert_2\Vert\mathbf{U}^*\Vert_{2,\infty}\Vert\mathbf{V}^*\Vert_{2,\infty},
	\end{split}
\end{equation}
where the first inequality follows from \textbf{Lemma \ref{t12}} with $\beta=2\alpha$, the second from \textbf{Lemma \ref{t11}} and the third from \textbf{Lemma \ref{t15}}. Next, we look at $\Vert\mathbf{U}_0\mathbf{\Sigma}_0\mathbf{V}_0^T-\mathbf{L}^*\Vert_2$:
\begin{equation}
	\label{eq:result2}
	\begin{split}
		&\quad\ \Vert\mathbf{U}_0\mathbf{\Sigma}_0\mathbf{V}_0^T-\mathbf{L}^*\Vert_2\\&\le\Vert\mathbf{U}_0\mathbf{\Sigma}_0\mathbf{V}_0^T-\mathbf{M}+\mathbf{S}_0\Vert_2+\Vert\mathbf{M}-\mathbf{S}_0-\mathbf{L}^*\Vert_2\\&\le2\Vert\mathbf{M}-\mathbf{S}_0-\mathbf{L}^*\Vert_2\\&\le8\alpha n\Vert\mathbf{\Sigma}^*\Vert_2\Vert\mathbf{U}^*\Vert_{2,\infty}\Vert\mathbf{V}^*\Vert_{2,\infty},
	\end{split}
\end{equation}
where we have used \textbf{Lemma \ref{t16}} and (\ref{eq:result1}).\\
In cases (i) and (iii), the condition $\alpha\le\frac{1}{16\kappa\mu_1 r}$ gives $\Vert\mathbf{U}_0\mathbf{\Sigma}_0\mathbf{V}_0^T-\mathbf{L}^*\Vert_2\le\frac{1}{2}\sigma_r(\mathbf{L}^*)$ and we have
\begin{equation}
	\begin{split}
	&\quad\ d(\mathbf{P}_0,\mathbf{Q}_0,\mathbf{P}^*,\mathbf{Q}^*)^2\\&
		=d(\mathbf{X}^T\mathbf{U}_0\mathbf{\Sigma}_0^{\frac{1}{2}},\mathbf{Y}^T\mathbf{V}_0\mathbf{\Sigma}_0^{\frac{1}{2}},\mathbf{X}^T\mathbf{U}^*\mathbf{\Sigma}^{*\frac{1}{2}}, \mathbf{Y}^T\mathbf{V}^*\mathbf{\Sigma}^{*\frac{1}{2}})^2\\&\le d(\mathbf{U}_0\mathbf{\Sigma}_0^{\frac{1}{2}},\mathbf{V}_0\mathbf{\Sigma}_0^{\frac{1}{2}},\mathbf{U}^*\mathbf{\Sigma}^{*\frac{1}{2}}, \mathbf{V}^*\mathbf{\Sigma}^{*\frac{1}{2}})^2\\&\le\frac{2}{\sqrt{2}-1}\frac{\Vert\mathbf{U}_0\mathbf{\Sigma}_0\mathbf{V}_0^T-\mathbf{L}^*\Vert_F^2}{\sigma_r(\mathbf{L}^*)}\\&\le\frac{2r}{\sqrt{2}-1}\frac{\Vert\mathbf{U}_0\mathbf{\Sigma}_0\mathbf{V}_0^T-\mathbf{L}^*\Vert_2^2}{\sigma_r(\mathbf{L}^*)}\le\frac{128r^3\alpha^2\kappa\sigma_1^*\mu_1^2}{\sqrt{2}-1},
	\end{split}
\end{equation}
using \textbf{Lemma \ref{t17}}, \textbf{Lemma \ref{t13}} and (\ref{eq:result2}). So, we have
\begin{equation}\label{result3}
	d(\mathbf{P}_0,\mathbf{Q}_0,\mathbf{P}^*,\mathbf{Q}^*)\le18\mu_1\alpha r\sqrt{r\kappa\sigma_1^*}.
\end{equation}
In case (ii), we have
\begin{equation}
	\Vert\mathbf{U}^*\Vert_{2,\infty}=\Vert\mathbf{X}\mathbf{X}^T\mathbf{U}^*\Vert_{2,\infty}\le\Vert\mathbf{X}^T\mathbf{U}^*\Vert_2\Vert\mathbf{X}\Vert_{2,\infty}\le \sqrt\frac{\mu_2d}{n},
\end{equation}
\begin{equation}
	\Vert\mathbf{V}^*\Vert_{2,\infty}=\Vert\mathbf{Y}\mathbf{Y}^T\mathbf{V}^*\Vert_{2,\infty}\le\Vert\mathbf{Y}^T\mathbf{V}^*\Vert_2\Vert\mathbf{Y}\Vert_{2,\infty}\le \sqrt\frac{\mu_2d}{n}.
\end{equation}
The condition $\alpha\le\frac{1}{16\kappa\mu_2 d}$ gives $\Vert\mathbf{U}_0\mathbf{\Sigma}_0\mathbf{V}_0^T-\mathbf{L}^*\Vert_2\le\frac{1}{2}\sigma_r(\mathbf{L}^*)$ and we have similar to (\ref{result3})
\begin{equation}
	d(\mathbf{P}_0,\mathbf{Q}_0,\mathbf{P}^*,\mathbf{Q}^*)\le18\mu_2\alpha d\sqrt{r\kappa\sigma_1^*}.
\end{equation}

\subsection{Proof of Theorem 2}
To ease our exposition, we define the following auxiliary quantities.\\\\
Let the solution set be
\begin{equation}
	\mathcal{E}=\{(\mathbf{A},\mathbf{B})\in\mathbb{R}^{d\times r}\times\mathbb{R}^{d\times r}|d(\mathbf{A},\mathbf{B},\mathbf{P}^*,\mathbf{Q}^*)=0\}.
\end{equation}
For any $(\mathbf{P},\mathbf{Q})\in\mathbb{R}^{d\times r}\times\mathbb{R}^{d\times r}$, the corresponding solution is given by
\begin{equation}
	(\mathbf{P}^\dagger,\mathbf{Q}^\dagger)\in\arg\min_{(\mathbf{A},\mathbf{B})\in\mathcal{E}}\Vert\mathbf{P}-\mathbf{A}\Vert_F^2+\Vert\mathbf{Q}-\mathbf{B}\Vert_F^2.
\end{equation}
Let $\Delta\mathbf{P}=\mathbf{P}-\mathbf{P}^\dagger$, $\Delta\mathbf{Q}=\mathbf{Q}-\mathbf{Q}^\dagger$ and $\delta=\Vert\Delta\mathbf{P}\Vert_F^2+\Vert\Delta\mathbf{Q}\Vert_F^2$, from which we have
\begin{equation}
	\begin{split}
		2\Vert\Delta\mathbf{P}\Vert_F\Vert\Delta\mathbf{Q}\Vert_F&\le\Vert\Delta\mathbf{P}\Vert_F^2+\Vert\Delta\mathbf{Q}\Vert_F^2,\\
		\Vert\Delta\mathbf{P}\Vert_F+\Vert\Delta\mathbf{Q}\Vert_F&\le\sqrt{2\delta},\\
		\Vert\Delta\mathbf{P}\Vert_F^2+\Vert\Delta\mathbf{Q}\Vert_F\Vert\Delta\mathbf{P}\Vert_F&\le\sqrt{2\delta}\Vert\Delta\mathbf{P}\Vert_F,\\
		\Vert\Delta\mathbf{Q}\Vert_F^2+\Vert\Delta\mathbf{Q}\Vert_F\Vert\Delta\mathbf{P}\Vert_F&\le\sqrt{2\delta}\Vert\Delta\mathbf{Q}\Vert_F,\\
		4\Vert\Delta\mathbf{Q}\Vert_F\Vert\Delta\mathbf{P}\Vert_F&\le\Vert\Delta\mathbf{P}\Vert_F^2+\Vert\Delta\mathbf{Q}\Vert_F^2+2\Vert\Delta\mathbf{Q}\Vert_F\Vert\Delta\mathbf{P}\Vert_F\\&\le\sqrt{2\delta}(\Vert\Delta\mathbf{Q}\Vert_F+\Vert\Delta\mathbf{P}\Vert_F).
	\end{split}
\end{equation}
Let $\mathcal{H}=\frac{1}{2}\Vert\mathbf{X}\mathbf{P}\mathbf{Q}^T\mathbf{Y}^T+\mathbf{S}-\mathbf{M}\Vert_F^2$ and $\Delta\mathbf{M}=\nabla_\mathbf{L}\mathcal{H}(\mathbf{P},\mathbf{Q})$, we have
\begin{equation}
	\nabla_\mathbf{L}\mathcal{H}(\mathbf{P},\mathbf{Q})=\mathbf{X}\mathbf{P}\mathbf{Q}^T\mathbf{Y}^T+\mathbf{S}-\mathbf{M}=\mathbf{L} + \mathbf{S}-\mathbf{L}^*-\mathbf{S}^*.
\end{equation}
We also have
\begin{equation}
	\nabla_\mathbf{P}\mathcal{H}(\mathbf{P},\mathbf{Q})=\mathbf{X}^T\nabla_\mathbf{L}\mathcal{H}\mathbf{Y}\mathbf{Q},
\end{equation}
\begin{equation}
	\nabla_\mathbf{Q}\mathcal{H}(\mathbf{P},\mathbf{Q})=(\mathbf{X}^T\nabla_\mathbf{L}\mathcal{H}\mathbf{Y})^T\mathbf{P}.
\end{equation}
Let $\mathcal{G}(\mathbf{P},\mathbf{Q})=\frac{1}{64}\Vert\mathbf{P}^T\mathbf{P}-\mathbf{Q}^T\mathbf{Q}\Vert_F^2$, we have
\begin{equation}
	\nabla_\mathbf{P}\mathcal{G}(\mathbf{P},\mathbf{Q})=\frac{1}{16}\mathbf{P}(\mathbf{P}^T\mathbf{P}-\mathbf{Q}^T\mathbf{Q}),
\end{equation}
\begin{equation}
	\nabla_\mathbf{Q}\mathcal{G}(\mathbf{P},\mathbf{Q})=\frac{1}{16}\mathbf{Q}(\mathbf{Q}^T\mathbf{Q}-\mathbf{P}^T\mathbf{P}).
\end{equation}
Let $\mathbf{F} = \begin{bmatrix}
\mathbf{P} \\
\mathbf{Q}
\end{bmatrix}$, $\mathbf{F}^\dagger = \begin{bmatrix}
\mathbf{P}^\dagger \\
\mathbf{Q}^\dagger
\end{bmatrix}$ and $\Delta\mathbf{F}=\mathbf{F}-\mathbf{F}^\dagger$, then we have $\delta=\Vert\Delta\mathbf{F}\Vert_F^2$.\\\\
We now state several lemmas that will help us construct the proof.\\
\begin{lemma}\label{t21}
	For any $\mathbf{P}\in\mathbb{R}^{d\times r}$ and $\mathbf{Q}\in\mathbb{R}^{d\times r}$, we have
	\begin{equation}
		\Vert\mathbf{L}-\mathbf{L}^*\Vert_F^2\le2\delta(\sqrt{\sigma_1^*}+\frac{\sqrt{2\delta}}{4})^2.
	\end{equation}
	\begin{proof}
		\begin{equation}
			\begin{split}
				\Vert\mathbf{L}-\mathbf{L}^*\Vert_F&=\Vert\mathbf{X}\mathbf{P}\mathbf{Q}^T\mathbf{Y}^T-\mathbf{X}\mathbf{P}^\dagger\mathbf{Q}^{\dagger T}\mathbf{Y}^T\Vert_F\\&=\Vert\mathbf{X}(\mathbf{P}^\dagger\Delta\mathbf{Q}^T+\Delta\mathbf{P}\mathbf{Q}^{\dagger T}+\Delta\mathbf{P}\Delta\mathbf{Q}^T)\mathbf{Y}^T\Vert_F\\&\le\Vert\mathbf{P}^\dagger\Delta\mathbf{Q}^T\Vert_F+\Vert\Delta\mathbf{P}\mathbf{Q}^{\dagger T}\Vert_F+\Vert\Delta\mathbf{P}\Delta\mathbf{Q}^T\Vert_F\\&\le\Vert\mathbf{P}^\dagger\Vert_2\Vert\Delta\mathbf{Q}\Vert_F+\Vert\Delta\mathbf{P}\Vert_F\Vert\mathbf{Q}^\dagger\Vert_2+\Vert\Delta\mathbf{P}\Vert_F\Vert\Delta\mathbf{Q}\Vert_F\\&\le \sqrt{\sigma_1^*}\Vert\Delta\mathbf{Q}\Vert_F+\sqrt{\sigma_1^*}\Vert\Delta\mathbf{P}\Vert_F+\frac{\sqrt{2\delta}}{4}(\Vert\Delta\mathbf{Q}\Vert_F+\Vert\Delta\mathbf{P}\Vert_F)\\&\le(\sqrt{\sigma_1^*}+\frac{\sqrt{2\delta}}{4})(\Vert\Delta\mathbf{Q}\Vert_F+\Vert\Delta\mathbf{P}\Vert_F)\le(\sqrt{\sigma_1^*}+\frac{\sqrt{2\delta}}{4})\sqrt{2\delta}.
			\end{split}
		\end{equation}
	\end{proof}
\end{lemma}
\begin{lemma}\label{t22}
	For $1\le i,j \le n$, in case (i), if $\Vert\mathbf{X}\mathbf{P}\Vert_{2,\infty}\le\sqrt\frac{3\mu_1r\sigma_1^*}{2n}$ and $\Vert\mathbf{Y}\mathbf{Q}\Vert_{2,\infty}\le\sqrt\frac{3\mu_1r\sigma_1^*}{2n}$, then
	\begin{equation}
		|(\mathbf{L}-\mathbf{L}^*)_{ij}|\le\frac{1}{2}\sqrt\frac{\mu_1r\sigma_1^*}{n}(3+\sqrt\frac{3}{2})(\Vert(\mathbf{X}\Delta\mathbf{P})_{i\cdot}\Vert_2+\Vert(\mathbf{Y}\Delta\mathbf{Q})_{j\cdot}\Vert_2)
	\end{equation}
	and in cases (ii) and (iii), if $\Vert\mathbf{X}\mathbf{P}\Vert_{2,\infty}\le \sqrt\frac{3\mu_2d\sigma_1^*}{2n}$ and $\Vert\mathbf{Y}\mathbf{Q}\Vert_{2,\infty}\le \sqrt\frac{3\mu_2d\sigma_1^*}{2n}$, then
	\begin{equation}
		|(\mathbf{L}-\mathbf{L}^*)_{ij}|\le\frac{1}{2}\sqrt\frac{\mu_2d\sigma_1^*}{n}(3+\sqrt\frac{3}{2})(\Vert(\mathbf{X}\Delta\mathbf{P})_{i\cdot}\Vert_2+\Vert(\mathbf{Y}\Delta\mathbf{Q})_{j\cdot}\Vert_2).
	\end{equation}
	\begin{proof}
		\begin{equation}
			\begin{split}
				&\quad\ |(\mathbf{L}-\mathbf{L}^*)_{ij}|\\&=|(\mathbf{X}\mathbf{P}\mathbf{Q}^T\mathbf{Y}^T-\mathbf{X}\mathbf{P}^\dagger\mathbf{Q}^{\dagger T}\mathbf{Y}^T)_{ij}|\\&\le|(\mathbf{X}\mathbf{P}^\dagger)_{i\cdot}(\mathbf{Y}\Delta\mathbf{Q})^T_{\cdot j}|+|(\mathbf{X}\Delta\mathbf{P})_{i\cdot}(\mathbf{Y}\mathbf{Q}^\dagger)^T_{\cdot j}|+|(\mathbf{X}\Delta\mathbf{P})_{i\cdot}(\mathbf{Y}\Delta\mathbf{Q})^T_{\cdot j}|\\&\le\Vert(\mathbf{X}\mathbf{P}^\dagger)_{i\cdot}\Vert_2\Vert(\mathbf{Y}\Delta\mathbf{Q})_{j\cdot}\Vert_2+\Vert(\mathbf{X}\Delta\mathbf{P})_{i\cdot}\Vert_2\Vert(\mathbf{Y}\mathbf{Q}^\dagger)_{j\cdot}\Vert_2+\Vert(\mathbf{X}\Delta\mathbf{P})_{i\cdot}\Vert_2\Vert(\mathbf{Y}\Delta\mathbf{Q})_{j\cdot}\Vert_2\\&\le\Vert\mathbf{X}\mathbf{P}^\dagger\Vert_{2,\infty}\Vert(\mathbf{Y}\Delta\mathbf{Q})_{j\cdot}\Vert_2+\Vert\mathbf{Y}\mathbf{Q}^\dagger\Vert_{2,\infty}\Vert(\mathbf{X}\Delta\mathbf{P})_{i\cdot}\Vert_2\\&\quad+\frac{1}{2}\Vert\mathbf{X}\Delta\mathbf{P}\Vert_{2,\infty}\Vert(\mathbf{Y}\Delta\mathbf{Q})_{j\cdot}\Vert_2+\frac{1}{2}\Vert\mathbf{Y}\Delta\mathbf{Q}\Vert_{2,\infty}\Vert(\mathbf{X}\Delta\mathbf{P})_{i\cdot}\Vert_2\\&\le\frac{1}{2}((3\Vert\mathbf{X}\mathbf{P}^\dagger\Vert_{2,\infty}+\Vert\mathbf{X}\mathbf{P}\Vert_{2,\infty})\Vert(\mathbf{Y}\Delta\mathbf{Q})_{j\cdot}\Vert_2\\&\quad+(3\Vert\mathbf{Y}\mathbf{Q}^\dagger\Vert_{2,\infty}+\Vert\mathbf{Y}\mathbf{Q}\Vert_{2,\infty})\Vert(\mathbf{X}\Delta\mathbf{P})_{i\cdot}\Vert_2),
			\end{split}
		\end{equation}
		where we have used $\Vert\mathbf{X}\Delta\mathbf{P}\Vert_{2,\infty}\le\Vert\mathbf{X}\mathbf{P}^\dagger\Vert_{2,\infty}+\Vert\mathbf{X}\mathbf{P}\Vert_{2,\infty}$ and $\Vert\mathbf{Y}\Delta\mathbf{Q}\Vert_{2,\infty}\le\Vert\mathbf{Y}\mathbf{Q}^\dagger\Vert_{2,\infty}+\Vert\mathbf{Y}\mathbf{Q}\Vert_{2,\infty}$.\\\\
		In case (i),
		\begin{equation}
			\begin{split}
				|(\mathbf{L}-\mathbf{L}^*)_{ij}|&\le\frac{1}{2}(3\sqrt\frac{\mu_1r\sigma_1^*}{n}+\sqrt\frac{3\mu_1r\sigma_1^*}{2n})(\Vert(\mathbf{X}\Delta\mathbf{P})_{i\cdot}\Vert_2+\Vert(\mathbf{Y}\Delta\mathbf{Q})_{j\cdot}\Vert_2)\\&=\frac{1}{2}\sqrt\frac{\mu_1r\sigma_1^*}{n}(3+\sqrt\frac{3}{2})(\Vert(\mathbf{X}\Delta\mathbf{P})_{i\cdot}\Vert_2+\Vert(\mathbf{Y}\Delta\mathbf{Q})_{j\cdot}\Vert_2).
			\end{split}
		\end{equation}
		In cases (ii) and (iii),
		\begin{equation}
			\begin{split}
				|(\mathbf{L}-\mathbf{L}^*)_{ij}|&\le\frac{1}{2}(3\sqrt\frac{\mu_2d\sigma_1^*}{n}+\sqrt\frac{3\mu_2d\sigma_1^*}{2n})(\Vert(\mathbf{X}\Delta\mathbf{P})_{i\cdot}\Vert_2+\Vert(\mathbf{Y}\Delta\mathbf{Q})_{j\cdot}\Vert_2)\\&=\frac{1}{2}\sqrt\frac{\mu_2d\sigma_1^*}{n}(3+\sqrt\frac{3}{2})(\Vert(\mathbf{X}\Delta\mathbf{P})_{i\cdot}\Vert_2+\Vert(\mathbf{Y}\Delta\mathbf{Q})_{j\cdot}\Vert_2).
			\end{split}
		\end{equation}
	\end{proof}
\end{lemma}
\begin{lemma}\label{t23}
	For any $\alpha\in(0,1)$, suppose the support index set $\Omega\subseteq[n]\times[n]$ satisfies $|\Omega_{i\cdot}|\le\alpha n$ for all $i\in[n]$ and $|\Omega_{\cdot j}|\le\alpha n$ for all $j\in[n]$ where $\Omega_{i\cdot}=\{(i,j)\in\Omega|j\in[n]\}$ and $\Omega_{\cdot j}=\{(i,j)\in\Omega|i\in[n]\}$. In case (i), we have
\end{lemma}
\begin{equation}
	\Vert\Pi_\Omega(\mathbf{L}-\mathbf{L}^*)\Vert_F^2\le\frac{\alpha\mu_1r\sigma_1^*}{2}(3+\sqrt\frac{3}{2})^2(\Vert\Delta\mathbf{P}\Vert_F^2+\Vert\Delta\mathbf{Q}\Vert_F^2),
\end{equation}
and in cases (ii) and (iii), we have
\begin{equation}
	\Vert\Pi_\Omega(\mathbf{L}-\mathbf{L}^*)\Vert_F^2\le\frac{\alpha\mu_2d\sigma_1^*}{2}(3+\sqrt\frac{3}{2})^2(\Vert\Delta\mathbf{P}\Vert_F^2+\Vert\Delta\mathbf{Q}\Vert_F^2).
\end{equation}
\begin{proof}
	\begin{equation}
		\Vert\Pi_\Omega(\mathbf{L}-\mathbf{L}^*)\Vert_F^2=\sum_{i,j\in\mathbf{\Omega}}|(\mathbf{L}-\mathbf{L}^*)_{ij}|^2.
	\end{equation}
	Using Lemma \ref{t22}, in case (i),
	\begin{equation}
		\begin{split}
			&\quad\ \sum_{i,j\in\mathbf{\Omega}}|(\mathbf{L}-\mathbf{L}^*)_{ij}|^2\\&\le\sum_{i,j\in\mathbf{\Omega}}\frac{\mu_1r\sigma_1^*}{4n}(3+\sqrt\frac{3}{2})^2(\Vert(\mathbf{X}\Delta\mathbf{P})_{i\cdot}\Vert_2+\Vert(\mathbf{Y}\Delta\mathbf{Q})_{j\cdot}\Vert_2)^2\\&\le\sum_{i,j\in\mathbf{\Omega}}\frac{\mu_1r\sigma_1^*}{2n}(3+\sqrt\frac{3}{2})^2(\Vert(\mathbf{X}\Delta\mathbf{P})_{i\cdot}\Vert_2^2+\Vert(\mathbf{Y}\Delta\mathbf{Q})_{j\cdot}\Vert_2^2)\\&\le\frac{\mu_1r\sigma_1^*}{2n}(3+\sqrt\frac{3}{2})^2(\sum_{i,j\in\mathbf{\Omega}}\Vert(\mathbf{X}\Delta\mathbf{P})_{i\cdot}\Vert_2^2+\sum_{i,j\in\mathbf{\Omega}}\Vert(\mathbf{Y}\Delta\mathbf{Q})_{j\cdot}\Vert_2^2)\\&\le\frac{\mu_1r\sigma_1^*}{2n}(3+\sqrt\frac{3}{2})^2(\sum_i\sum_{j\in\mathbf{\Omega}_{i\cdot}}\Vert(\mathbf{X}\Delta\mathbf{P})_{i\cdot}\Vert_2^2+\sum_j\sum_{i\in\mathbf{\Omega}_{\cdot j}}\Vert(\mathbf{Y}\Delta\mathbf{Q})_{j\cdot}\Vert_2^2)\\&\le\frac{\alpha\mu_1r\sigma_1^*}{2}(3+\sqrt\frac{3}{2})^2(\Vert\mathbf{X}\Delta\mathbf{P}\Vert_F^2+\Vert\mathbf{Y}\Delta\mathbf{Q}\Vert_F^2)\\&\le\frac{\alpha\mu_1r\sigma_1^*}{2}(3+\sqrt\frac{3}{2})^2(\Vert\Delta\mathbf{P}\Vert_F^2+\Vert\Delta\mathbf{Q}\Vert_F^2).
		\end{split}
	\end{equation}
	and in cases (ii) and (iii),
	\begin{equation}
		\begin{split}
			&\quad\ \sum_{i,j\in\mathbf{\Omega}}|(\mathbf{L}-\mathbf{L}^*)_{ij}|^2\\&\le\sum_{i,j\in\mathbf{\Omega}}\frac{\mu_2d\sigma_1^*}{4n}(3+\sqrt\frac{3}{2})^2(\Vert(\mathbf{X}\Delta\mathbf{P})_{i\cdot}\Vert_2+\Vert(\mathbf{Y}\Delta\mathbf{Q})_{j\cdot}\Vert_2)^2\\&\le\sum_{i,j\in\mathbf{\Omega}}\frac{\mu_2d\sigma_1^*}{2n}(3+\sqrt\frac{3}{2})^2(\Vert(\mathbf{X}\Delta\mathbf{P})_{i\cdot}\Vert_2^2+\Vert(\mathbf{Y}\Delta\mathbf{Q})_{j\cdot}\Vert_2^2)\\&\le\frac{\mu_2d\sigma_1^*}{2n}(3+\sqrt\frac{3}{2})^2(\sum_{i,j\in\mathbf{\Omega}}\Vert(\mathbf{X}\Delta\mathbf{P})_{i\cdot}\Vert_2^2+\sum_{i,j\in\mathbf{\Omega}}\Vert(\mathbf{Y}\Delta\mathbf{Q})_{j\cdot}\Vert_2^2)\\&\le\frac{\mu_2d\sigma_1^*}{2n}(3+\sqrt\frac{3}{2})^2(\sum_i\sum_{j\in\mathbf{\Omega}_{i\cdot}}\Vert(\mathbf{X}\Delta\mathbf{P})_{i\cdot}\Vert_2^2+\sum_j\sum_{i\in\mathbf{\Omega}_{\cdot j}}\Vert(\mathbf{Y}\Delta\mathbf{Q})_{j\cdot}\Vert_2^2)\\&\le\frac{\alpha\mu_2d\sigma_1^*}{2}(3+\sqrt\frac{3}{2})^2(\Vert\mathbf{X}\Delta\mathbf{P}\Vert_F^2+\Vert\mathbf{Y}\Delta\mathbf{Q}\Vert_F^2)\\&\le\frac{\alpha\mu_2d\sigma_1^*}{2}(3+\sqrt\frac{3}{2})^2(\Vert\Delta\mathbf{P}\Vert_F^2+\Vert\Delta\mathbf{Q}\Vert_F^2).
		\end{split}
	\end{equation}
\end{proof}
\begin{lemma}\label{t24}
	Given that $\mathbf{S}=\mathcal{T}_{\alpha+\min(10\alpha,0.1)}(\mathbf{M}-\mathbf{X}\mathbf{P}\mathbf{Q}^T\mathbf{Y}^T)$, we have in case (i)
	\begin{multline}
		\langle\mathbf{X}^T\nabla_\mathbf{L}\mathcal{H}(\mathbf{P},\mathbf{Q})\mathbf{Y},\mathbf{P}\mathbf{Q}^T-\mathbf{P}^\dagger\mathbf{Q}^{\dagger T}+\Delta\mathbf{P}\Delta\mathbf{Q}^T\rangle\ge\Vert\mathbf{L}-\mathbf{L}^*\Vert_F^2\\-\frac{\mu_1r\sigma_1^*\delta}{4}((4+\beta)\alpha+2\min(10\alpha,0.1))(3+\sqrt\frac{3}{2})^2-\frac{2\alpha\delta}{\beta\min(10\alpha,0.1)}(\sqrt{\sigma_1^*}+\frac{\sqrt{2\delta}}{4})^2\\-\frac{\sqrt{2}+2\sqrt{\frac{\alpha}{\min(10\alpha,0.1)}}}{2}\sqrt{\delta^3}(\sqrt{\sigma_1^*}+\frac{\sqrt{2\delta}}{4}),
	\end{multline}
	and in cases (ii) and (iii)
	\begin{multline}
		\langle\mathbf{X}^T\nabla_\mathbf{L}\mathcal{H}(\mathbf{P},\mathbf{Q})\mathbf{Y},\mathbf{P}\mathbf{Q}^T-\mathbf{P}^\dagger\mathbf{Q}^{\dagger T}+\Delta\mathbf{P}\Delta\mathbf{Q}^T\rangle\ge\Vert\mathbf{L}-\mathbf{L}^*\Vert_F^2\\-\frac{\mu_2d\sigma_1^*\delta}{4}((4+\beta)\alpha+2\min(10\alpha, 0.1))(3+\sqrt\frac{3}{2})^2-\frac{2\alpha\delta}{\beta\min(10\alpha,0.1)}(\sqrt{\sigma_1^*}+\frac{\sqrt{2\delta}}{4})^2\\-\frac{\sqrt{2}+2\sqrt{\frac{\alpha}{\min(10\alpha,0.1)}}}{2}\sqrt{\delta^3}(\sqrt{\sigma_1^*}+\frac{\sqrt{2\delta}}{4}).
	\end{multline}
\end{lemma}
\begin{proof}
	\begin{equation}
		\begin{split}
			&\quad\ \langle\mathbf{X}^T\nabla_\mathbf{L}\mathcal{H}(\mathbf{P},\mathbf{Q})\mathbf{Y},\mathbf{P}\mathbf{Q}^T-\mathbf{P}^\dagger\mathbf{Q}^{\dagger T}+\Delta\mathbf{P}\Delta\mathbf{Q}^T\rangle\\&=\langle\mathbf{L}+\mathbf{S}-\mathbf{L}^*-\mathbf{S}^*,\mathbf{L}-\mathbf{L}^*+\mathbf{X}\Delta\mathbf{P}\Delta\mathbf{Q}^T\mathbf{Y}^T\rangle\\&\ge\Vert\mathbf{L}-\mathbf{L}^*\Vert_F^2-|\langle\mathbf{S}-\mathbf{S}^*,\mathbf{L}-\mathbf{L}^*\rangle|-|\langle\mathbf{L}+\mathbf{S}-\mathbf{L}^*-\mathbf{S}^*,\mathbf{X}\Delta\mathbf{P}\Delta\mathbf{Q}^T\mathbf{Y}^T\rangle|.
		\end{split}
	\end{equation}
	Following \citet{yi16} lemma 2, we have
	\begin{equation}
		|\langle\mathbf{S}-\mathbf{S}^*,\mathbf{L}-\mathbf{L}^*\rangle|\le\Vert\mathbf{\Pi}_\mathbf{\Omega}(\mathbf{L}-\mathbf{L}^*)\Vert_F^2+(1+\frac{\beta}{2})\Vert\mathbf{\Pi}_{\mathbf{\Omega}^*\backslash\mathbf{\Omega}}(\mathbf{L}-\mathbf{L}^*)\Vert_F^2+\frac{\alpha}{\beta\min(10\alpha,0.1)}\Vert\mathbf{L}-\mathbf{L}^*\Vert_F^2,
	\end{equation}
	where $\beta>0$, $\mathbf{\Omega}$ and $\mathbf{\Omega}^*$ are supports of $\mathbf{S}$ and $\mathbf{S}^*$ respectively.\\\\
	On the other hand,
	\begin{multline}
		|\langle\mathbf{L}+\mathbf{S}-\mathbf{L}^*-\mathbf{S}^*,\mathbf{X}\Delta\mathbf{P}\Delta\mathbf{Q}^T\mathbf{Y}^T\rangle|\le|\langle\mathbf{\Pi}_{\mathbf{\Omega}^{*c}\cap\mathbf{\Omega}^c}(\mathbf{L}-\mathbf{L}^*),\mathbf{X}\Delta\mathbf{P}\Delta\mathbf{Q}^T\mathbf{Y}^T\rangle|\\+|\langle\mathbf{\Pi}_{\mathbf{\Omega^{*}\cap\mathbf{\Omega}^c}}(\Delta\mathbf{M}),\mathbf{X}\Delta\mathbf{P}\Delta\mathbf{Q}^T\mathbf{Y}^T\rangle|,
	\end{multline}
	because $\Delta\mathbf{M}$ has support $\mathbf{\Omega}^c$.
	From Cauchy-Swartz inequality, we have
	\begin{equation}
		\begin{split}
			&\quad\ |\langle\mathbf{\Pi}_{\mathbf{\Omega}^{*c}\cap\mathbf{\Omega}^c}(\mathbf{L}-\mathbf{L}^*),\mathbf{X}\Delta\mathbf{P}\Delta\mathbf{Q}^T\mathbf{Y}^T\rangle|\\&\le\Vert\mathbf{\Pi}_{\mathbf{\Omega}^{*c}\cap\mathbf{\Omega}^c}(\mathbf{L}-\mathbf{L}^*)\Vert_F\Vert\mathbf{X}\Delta\mathbf{P}\Delta\mathbf{Q}^T\mathbf{Y}^T\Vert_F\\&\le\Vert\mathbf{L}-\mathbf{L}^*\Vert_F\Vert\mathbf{X}\Delta\mathbf{P}\Delta\mathbf{Q}^T\mathbf{Y}^T\Vert_F\\&\le\Vert\mathbf{L}-\mathbf{L}^*\Vert_F\Vert\Delta\mathbf{P}\Vert_F\Vert\Delta\mathbf{Q}\Vert_F\\&\le\frac{\delta}{2}\Vert\mathbf{L}-\mathbf{L}^*\Vert_F.
		\end{split}
	\end{equation}
	From \citet{yi16} lemma 2, we have
	\begin{equation}
		\begin{split}
			|\langle\mathbf{\Pi}_{\mathbf{\Omega^{*}\cap\mathbf{\Omega}^c}}(\Delta\mathbf{M}),\mathbf{X}\Delta\mathbf{P}\Delta\mathbf{Q}^T\mathbf{Y}^T\rangle|&\le\sqrt{\frac{2\alpha}{\min(10\alpha,0.1)}}\Vert\mathbf{L}-\mathbf{L}^*\Vert_F\Vert\mathbf{X}\Delta\mathbf{P}\Delta\mathbf{Q}^T\mathbf{Y}^T\Vert_F\\&\le\delta\sqrt{\frac{\alpha}{2\min(10\alpha,0.1)}}\Vert\mathbf{L}-\mathbf{L}^*\Vert_F.
		\end{split}
	\end{equation}
	So,
	\begin{equation}
		\begin{split}
			|\langle\mathbf{L}+\mathbf{S}-\mathbf{L}^*-\mathbf{S}^*,\mathbf{X}\Delta\mathbf{P}\Delta\mathbf{Q}^T\mathbf{Y}^T\rangle|&\le\frac{\delta}{2}\Vert\mathbf{L}-\mathbf{L}^*\Vert_F+\delta\sqrt{\frac{\alpha}{2\min(10\alpha,0.1)}}\Vert\mathbf{L}-\mathbf{L}^*\Vert_F\\&\le\frac{\delta}{2}(1+\sqrt{\frac{2\alpha}{\min(10\alpha,0.1)}})\Vert\mathbf{L}-\mathbf{L}^*\Vert_F.
		\end{split}
	\end{equation}
	Together, we have
	\begin{multline}
		\langle\mathbf{X}^T\nabla_\mathbf{L}\mathcal{H}(\mathbf{P},\mathbf{Q})\mathbf{Y},\mathbf{P}\mathbf{Q}^T-\mathbf{P}^\dagger\mathbf{Q}^{\dagger T}+\Delta\mathbf{P}\Delta\mathbf{Q}^T\rangle\ge\Vert\mathbf{L}-\mathbf{L}^*\Vert_F^2-\Vert\mathbf{\Pi}_\mathbf{\Omega}(\mathbf{L}-\mathbf{L}^*)\Vert_F^2\\-(1+\frac{\beta}{2})\Vert\mathbf{\Pi}_{\mathbf{\Omega}^*\backslash\mathbf{\Omega}}(\mathbf{L}-\mathbf{L}^*)\Vert_F^2-\frac{\alpha}{\beta\min(10\alpha,0.1)}\Vert\mathbf{L}-\mathbf{L}^*\Vert_F^2-\frac{\delta}{2}(1+\sqrt{\frac{2\alpha}{\min(10\alpha,0.1)}})\Vert\mathbf{L}-\mathbf{L}^*\Vert_F.
	\end{multline}
	From \textbf{Lemma \ref{t21}}, we have
	\begin{multline}
		\langle\mathbf{X}^T\nabla_\mathbf{L}\mathcal{H}(\mathbf{P},\mathbf{Q})\mathbf{Y},\mathbf{P}\mathbf{Q}^T-\mathbf{P}^\dagger\mathbf{Q}^{\dagger T}+\Delta\mathbf{P}\Delta\mathbf{Q}^T\rangle\ge\Vert\mathbf{L}-\mathbf{L}^*\Vert_F^2-\Vert\mathbf{\Pi}_\mathbf{\Omega}(\mathbf{L}-\mathbf{L}^*)\Vert_F^2\\-(1+\frac{\beta}{2})\Vert\mathbf{\Pi}_{\mathbf{\Omega}^*\backslash\mathbf{\Omega}}(\mathbf{L}-\mathbf{L}^*)\Vert_F^2-\frac{2\alpha\delta}{\beta\min(10\alpha,0.1)}(\sqrt{\sigma_1^*}+\frac{\sqrt{2\delta}}{4})^2-\frac{\sqrt{2}+2\sqrt{\frac{\alpha}{\min(10\alpha,0.1)}}}{2}\sqrt{\delta^3}(\sqrt{\sigma_1^*}+\frac{\sqrt{2\delta}}{4}).
	\end{multline}
	Since $\mathbf{\Pi}_\mathbf{\Omega}(\mathbf{L}-\mathbf{L}^*)$ and $\mathbf{\Pi}_{\mathbf{\Omega}^*\backslash\mathbf{\Omega}}(\mathbf{L}-\mathbf{L}^*)$ have at most $\alpha+\min(10\alpha,0.1)$-fraction and $\alpha$-fraction non-zero entries per row and column respectively, from \textbf{Lemma \ref{t23}}, we have in case (i) 
	\begin{equation}
		\begin{split}
			&\quad\ \langle\mathbf{X}^T\nabla_\mathbf{L}\mathcal{H}(\mathbf{P},\mathbf{Q})\mathbf{Y},\mathbf{P}\mathbf{Q}^T-\mathbf{P}^\dagger\mathbf{Q}^{\dagger T}+\Delta\mathbf{P}\Delta\mathbf{Q}^T\rangle\\&\ge\Vert\mathbf{L}-\mathbf{L}^*\Vert_F^2-\frac{\alpha\mu_1r\sigma_1^*}{2}(3+\sqrt\frac{3}{2})^2(\Vert\Delta\mathbf{P}\Vert_F^2+\Vert\Delta\mathbf{Q}\Vert_F^2)\\&\quad-\frac{\min(10\alpha,0.1)\mu_1r\sigma_1^*}{2}(3+\sqrt\frac{3}{2})^2(\Vert\Delta\mathbf{P}\Vert_F^2+\Vert\Delta\mathbf{Q}\Vert_F^2)\\&\quad-\frac{\alpha\mu_1r\sigma_1^*}{2}(1+\frac{\beta}{2})(3+\sqrt\frac{3}{2})^2(\Vert\Delta\mathbf{P}\Vert_F^2+\Vert\Delta\mathbf{Q}\Vert_F^2)\\&\quad-\frac{2\alpha\delta}{\beta\min(10\alpha,0.1)}(\sqrt{\sigma_1^*}+\frac{\sqrt{2\delta}}{4})^2-\frac{\sqrt{2}+2\sqrt{\frac{\alpha}{\min(10\alpha,0.1)}}}{2}\sqrt{\delta^3}(\sqrt{\sigma_1^*}+\frac{\sqrt{2\delta}}{4})\\&\ge\Vert\mathbf{L}-\mathbf{L}^*\Vert_F^2-\frac{\mu_1r\sigma_1^*\delta}{4}((4+\beta)\alpha+2\min(10\alpha,0.1))(3+\sqrt\frac{3}{2})^2\\&\quad-\frac{2\alpha\delta}{\beta\min(10\alpha,0.1)}(\sqrt{\sigma_1^*}+\frac{\sqrt{2\delta}}{4})^2-\frac{\sqrt{2}+2\sqrt{\frac{\alpha}{\min(10\alpha,0.1)}}}{2}\sqrt{\delta^3}(\sqrt{\sigma_1^*}+\frac{\sqrt{2\delta}}{4}),
		\end{split}
	\end{equation}
	and in cases (ii) and (iii)
	\begin{equation}
		\begin{split}
			&\quad\ \langle\mathbf{X}^T\nabla_\mathbf{L}\mathcal{H}(\mathbf{P},\mathbf{Q})\mathbf{Y},\mathbf{P}\mathbf{Q}^T-\mathbf{P}^\dagger\mathbf{Q}^{\dagger T}+\Delta\mathbf{P}\Delta\mathbf{Q}^T\rangle\\&\ge\Vert\mathbf{L}-\mathbf{L}^*\Vert_F^2-\frac{\alpha\mu_2d\sigma_1^*}{2}(3+\sqrt\frac{3}{2})^2(\Vert\Delta\mathbf{P}\Vert_F^2+\Vert\Delta\mathbf{Q}\Vert_F^2)\\&\quad-\frac{\min(10\alpha, 0.1)\mu_2d\sigma_1^*}{2}(3+\sqrt\frac{3}{2})^2(\Vert\Delta\mathbf{P}\Vert_F^2+\Vert\Delta\mathbf{Q}\Vert_F^2)\\&\quad-\frac{\alpha\mu_2d\sigma_1^*}{2}(1+\frac{\beta}{2})(3+\sqrt\frac{3}{2})^2(\Vert\Delta\mathbf{P}\Vert_F^2+\Vert\Delta\mathbf{Q}\Vert_F^2)\\&\quad-\frac{2\min(10\alpha,0.1)\alpha\delta}{\beta}(\sqrt{\sigma_1^*}+\frac{\sqrt{2\delta}}{4})^2-\frac{\sqrt{2}+2\sqrt{\frac{\alpha}{\min(10\alpha,0.1)}}}{2}\sqrt{\delta^3}(\sqrt{\sigma_1^*}+\frac{\sqrt{2\delta}}{4})\\&\ge\Vert\mathbf{L}-\mathbf{L}^*\Vert_F^2-\frac{\mu_2d\sigma_1^*\delta}{4}((4+\beta)\alpha+2\min(10\alpha,0.1))(3+\sqrt\frac{3}{2})^2\\&\quad-\frac{2\min(10\alpha,0.1)\alpha\delta}{\beta}(\sqrt{\sigma_1^*}+\frac{\sqrt{2\delta}}{4})^2-\frac{\sqrt{2}+2\sqrt{\frac{\alpha}{\min(10\alpha,0.1)}}}{2}\sqrt{\delta^3}(\sqrt{\sigma_1^*}+\frac{\sqrt{2\delta}}{4}).
		\end{split}
	\end{equation}
\end{proof}
\begin{lemma}\label{t25}
	When $\Vert\mathbf{F}-\mathbf{F}^\dagger\Vert_2\le\sqrt{2\sigma_r^*}$, given that $\Vert\mathbf{P}\Vert_2\le\sqrt\frac{3\sigma_1^*}{2}$ and $\Vert\mathbf{Q}\Vert_2\le\sqrt\frac{3\sigma_1^*}{2}$ we have
	\begin{multline}
		\langle\nabla_\mathbf{P}\mathcal{G}(\mathbf{P},\mathbf{Q}),\mathbf{P}-\mathbf{P}^\dagger\rangle+\langle\nabla_\mathbf{Q}\mathcal{G}(\mathbf{P},\mathbf{Q}),\mathbf{Q}-\mathbf{Q}^\dagger\rangle\ge\frac{1}{64}\Vert\mathbf{P}^T\mathbf{P}-\mathbf{Q}^T\mathbf{Q}\Vert_F^2\\+\frac{1}{64}(2\sqrt{\sigma_r^*\delta} - \delta)^2-\frac{1}{16}\Vert\mathbf{L}-\mathbf{L}^*\Vert_F^2-\frac{\sqrt{2}+\sqrt{3}}{32}\sqrt{\sigma_1^*\delta^3}.
	\end{multline}
\end{lemma}
\begin{proof}
	\begin{equation}
		\begin{split}
			\mathbf{P}^{\dagger T}\mathbf{P}^\dagger&=(\mathbf{X}^T\mathbf{U}^*\mathbf{\Sigma}^{*\frac{1}{2}}\mathbf{R})^T(\mathbf{X}^T\mathbf{U}^*\mathbf{\Sigma}^{*\frac{1}{2}}\mathbf{R})\\&=\mathbf{R}^T\mathbf{\Sigma}^{*\frac{1}{2}T}\mathbf{U}^{*T}\mathbf{X}\mathbf{X}^T\mathbf{U}^*\mathbf{\Sigma}^{*\frac{1}{2}}\mathbf{R}\\&=\mathbf{R}^T\mathbf{\Sigma}^{*\frac{1}{2}T}\mathbf{U}^{*T}\mathbf{U}^*\mathbf{\Sigma}^{*\frac{1}{2}}\mathbf{R}\\&=\mathbf{R}^T\mathbf{\Sigma}^{*\frac{1}{2}T}\mathbf{\Sigma}^{*\frac{1}{2}}\mathbf{R}\\&=\mathbf{R}^T\mathbf{\Sigma}^{*\frac{1}{2}T}\mathbf{V}^{*T}\mathbf{V}^*\mathbf{\Sigma}^{*\frac{1}{2}}\mathbf{R}\\&=\mathbf{R}^T\mathbf{\Sigma}^{*\frac{1}{2}T}\mathbf{V}^{*T}\mathbf{Y}\mathbf{Y}^T\mathbf{V}^*\mathbf{\Sigma}^{*\frac{1}{2}}\mathbf{R}\\&=(\mathbf{Y}^T\mathbf{V}^*\mathbf{\Sigma}^{*\frac{1}{2}}\mathbf{R})^T(\mathbf{Y}^T\mathbf{V}^*\mathbf{\Sigma}^{*\frac{1}{2}}\mathbf{R})\\&=\mathbf{Q}^{\dagger T}\mathbf{Q}^\dagger.
		\end{split}
	\end{equation}
	Then, following \citet{yi16} lemma 3, we have
	\begin{multline}
		\langle\nabla_\mathbf{P}\mathcal{G}(\mathbf{P},\mathbf{Q}),\mathbf{P}-\mathbf{P}^\dagger\rangle+\langle\nabla_\mathbf{Q}\mathcal{G}(\mathbf{P},\mathbf{Q}),\mathbf{Q}-\mathbf{Q}^\dagger\rangle\\=\frac{1}{32}\Vert\mathbf{P}^T\mathbf{P}-\mathbf{Q}^T\mathbf{Q}\Vert_F^2+\frac{1}{32}\langle\mathbf{P}^T\mathbf{P}-\mathbf{Q}^T\mathbf{Q},\Delta\mathbf{P}^T\Delta\mathbf{P}-\Delta\mathbf{Q}^T\Delta\mathbf{Q}\rangle.
	\end{multline}
	\begin{equation}
		\begin{split}
			&\quad\ \frac{1}{32}\langle\mathbf{P}^T\mathbf{P}-\mathbf{Q}^T\mathbf{Q},\Delta\mathbf{P}^T\Delta\mathbf{P}-\Delta\mathbf{Q}^T\Delta\mathbf{Q}\rangle\\&\le\frac{1}{32}|\langle\mathbf{P}^T\mathbf{P}-\mathbf{Q}^T\mathbf{Q},\Delta\mathbf{P}^T\Delta\mathbf{P}-\Delta\mathbf{Q}^T\Delta\mathbf{Q}\rangle|\\&\le\frac{1}{32}\Vert\mathbf{P}^T\mathbf{P}-\mathbf{Q}^T\mathbf{Q}\Vert_F\Vert\Delta\mathbf{P}^T\Delta\mathbf{P}-\Delta\mathbf{Q}^T\Delta\mathbf{Q}\Vert_F\\&\le\frac{1}{32}\Vert\mathbf{P}^T\mathbf{P}-\mathbf{Q}^T\mathbf{Q}\Vert_F(\Vert\Delta\mathbf{P}\Vert_F^2+\Vert\Delta\mathbf{Q}\Vert_F^2)\\&\le\frac{1}{32}\Vert\mathbf{P}^T\mathbf{P}-\mathbf{P}^{\dagger T}\mathbf{P}^\dagger+\mathbf{Q}^{\dagger T}\mathbf{Q}^\dagger-\mathbf{Q}^T\mathbf{Q}\Vert_F\delta\\&\le\frac{1}{32}(\Vert\mathbf{P}^T\mathbf{P}-\mathbf{P}^{\dagger T}\mathbf{P}^\dagger\Vert_F+\Vert\mathbf{Q}^T\mathbf{Q}-\mathbf{Q}^{\dagger T}\mathbf{Q}^\dagger\Vert_F)\delta\\&\le\frac{1}{32}(\Vert\mathbf{P}^T\mathbf{P}-\mathbf{P}^T\mathbf{P}^\dagger+\mathbf{P}^T\mathbf{P}^\dagger-\mathbf{P}^{\dagger T}\mathbf{P}^\dagger\Vert_F+\Vert\mathbf{Q}^T\mathbf{Q}-\mathbf{Q}^T\mathbf{Q}^\dagger+\mathbf{Q}^T\mathbf{Q}^\dagger-\mathbf{Q}^{\dagger T}\mathbf{Q}^\dagger\Vert_F)\delta\\&\le\frac{1}{32}(\Vert\mathbf{P}^T\Delta\mathbf{P}+\Delta\mathbf{P}^T\mathbf{P}^\dagger\Vert_F+\Vert\mathbf{Q}^T\Delta\mathbf{Q}+\Delta\mathbf{Q}^T\mathbf{Q}^\dagger\Vert_F)\delta\\&\le\frac{1}{32}((\Vert\mathbf{P}\Vert_2+\Vert\mathbf{P}^\dagger\Vert_2)\Vert\Delta\mathbf{P}\Vert_F+(\Vert\mathbf{Q}\Vert_2+\Vert\mathbf{Q}^\dagger\Vert_2)\Vert\Delta\mathbf{Q}\Vert_F)\delta\\&\le\frac{1}{32}(\sqrt{\sigma_1^*}+\sqrt\frac{3\sigma_1^*}{2})(\Vert\Delta P\Vert_F+\Vert\Delta Q\Vert_F)\delta\\&\le\frac{\sqrt{2}+\sqrt{3}}{32}\sqrt{\sigma_1^*\delta^3}.
		\end{split}
	\end{equation}
	Following \citet{yi16} lemma 3, we have
	\begin{equation}
		\frac{1}{32}\Vert\mathbf{P}^T\mathbf{P}-\mathbf{Q}^T\mathbf{Q}\Vert_F^2\ge\frac{1}{64}\Vert\mathbf{P}^T\mathbf{P}-\mathbf{Q}^T\mathbf{Q}\Vert_F^2+\frac{1}{64}(\sqrt{2}\Vert\Delta\mathbf{F}\mathbf{F}^{\dagger T}\Vert_F - \delta)^2-\frac{1}{16}\Vert\mathbf{L}-\mathbf{L}^*\Vert_F^2,
	\end{equation}
	where we have used the fact that $-\Vert\mathbf{P}\mathbf{Q}^T-\mathbf{P}^\dagger\mathbf{Q}^{\dagger T}\Vert_F^2\ge-\Vert\mathbf{X}\mathbf{P}\mathbf{Q}^T\mathbf{Y}^T-\mathbf{X}\mathbf{P}^\dagger\mathbf{Q}^{\dagger T}\mathbf{Y}^T\Vert_F^2.$\\
	We know that $\mathbf{F}^\dagger = \begin{bmatrix}
	\mathbf{P}^\dagger \\
	\mathbf{Q}^\dagger
	\end{bmatrix}= \begin{bmatrix}
	\mathbf{X}^T\mathbf{U}^* \\
	\mathbf{Y}^T\mathbf{V}^*
	\end{bmatrix}\mathbf{\Sigma}^{*\frac{1}{2}}\mathbf{R}$. If we let $\mathbf{E} = \begin{bmatrix}
	\mathbf{X}^T\mathbf{U}^* \\
	\mathbf{Y}^T\mathbf{V}^*
	\end{bmatrix}$, then $\mathbf{E}^T\mathbf{E}=[\mathbf{U}^{*T}X\ \ \mathbf{V}^{*T}Y]\begin{bmatrix}
	\mathbf{X}^T\mathbf{U}^* \\
	\mathbf{Y}^T\mathbf{V}^*
	\end{bmatrix}=2\mathbf{I}\in\mathbb{R}^{r\times r}$. So
	\begin{equation}
		\mathbf{F}^\dagger = (\frac{\sqrt{2}}{2}\begin{bmatrix}
			\mathbf{X}^T\mathbf{U}^* \\
			\mathbf{Y}^T\mathbf{V}^*
		\end{bmatrix})(\sqrt{2}\mathbf{\Sigma}^{*\frac{1}{2}})\mathbf{R},
	\end{equation}
	is the SVD of $\mathbf{F}^\dagger$. Therefore,
	\begin{equation}
		\frac{1}{32}\Vert\mathbf{P}^T\mathbf{P}-\mathbf{Q}^T\mathbf{Q}\Vert_F^2\ge\frac{1}{64}\Vert\mathbf{P}^T\mathbf{P}-\mathbf{Q}^T\mathbf{Q}\Vert_F^2+\frac{1}{64}(2\sqrt{\sigma_r^*\delta} - \delta)^2-\frac{1}{16}\Vert\mathbf{L}-\mathbf{L}^*\Vert_F^2.
	\end{equation}
	Thus, altogether we have
	\begin{multline}
		\langle\nabla_\mathbf{P}\mathcal{G}(\mathbf{P},\mathbf{Q}),\mathbf{P}-\mathbf{P}^\dagger\rangle+\langle\nabla_\mathbf{Q}\mathcal{G}(\mathbf{P},\mathbf{Q}),\mathbf{Q}-\mathbf{Q}^\dagger\rangle\ge\frac{1}{64}\Vert\mathbf{P}^T\mathbf{P}-\mathbf{Q}^T\mathbf{Q}\Vert_F^2\\+\frac{1}{64}(2\sqrt{\sigma_r^*\delta} - \delta)^2-\frac{1}{16}\Vert\mathbf{L}-\mathbf{L}^*\Vert_F^2-\frac{\sqrt{2}+\sqrt{3}}{32}\sqrt{\sigma_1^*\delta^3}.
	\end{multline}
\end{proof}
\begin{lemma}\label{t26}
	When $\mathbf{S}=\mathcal{T}_{\alpha+\min(10\alpha,0.1)}(\mathbf{M}-\mathbf{X}\mathbf{P}\mathbf{Q}^T\mathbf{Y}^T)$, given that $\Vert\mathbf{P}\Vert_2\le\sqrt\frac{3\sigma_1^*}{2}$ and $\Vert\mathbf{Q}\Vert_2\le\sqrt\frac{3\sigma_1^*}{2}$, we have
	\begin{equation}
		\Vert\nabla_\mathbf{L}\mathcal{H}(\mathbf{P},\mathbf{Q})\Vert_F^2\le(1+\sqrt{\frac{2\alpha}{\min(10\alpha,0.1)}})^2\Vert\mathbf{L}-\mathbf{L}^*\Vert_F^2,
	\end{equation}
	\begin{equation}
		\Vert\nabla_\mathbf{P}\mathcal{G}(\mathbf{P},\mathbf{Q})\Vert_F^2+\Vert\nabla_\mathbf{Q}\mathcal{G}(\mathbf{P},\mathbf{Q})\Vert_F^2\le\frac{3\sigma_1^*}{256}\Vert\mathbf{P}^T\mathbf{P}-\mathbf{Q}^T\mathbf{Q}\Vert_F^2.
	\end{equation}
\end{lemma}
\begin{proof}
	\begin{equation}
		\begin{split}
			&\quad\ \Vert\nabla_\mathbf{P}\mathcal{G}(\mathbf{P},\mathbf{Q})\Vert_F^2+\Vert\nabla_\mathbf{Q}\mathcal{G}(\mathbf{P},\mathbf{Q})\Vert_F^2\\&=\Vert\frac{1}{16}\mathbf{P}(\mathbf{P}^T\mathbf{P}-\mathbf{Q}^T\mathbf{Q})\Vert_F^2+\Vert\frac{1}{16}\mathbf{Q}(\mathbf{Q}^T\mathbf{Q}-\mathbf{P}^T\mathbf{P})\Vert_F^2\\&\le\frac{1}{256}(\Vert\mathbf{P}\Vert_2^2+\Vert\mathbf{Q}\Vert_2^2)\Vert\mathbf{P}^T\mathbf{P}-\mathbf{Q}^T\mathbf{Q}\Vert_F^2\\&\le\frac{1}{256}(\frac{3\sigma_1^*}{2}+\frac{3\sigma_1^*}{2})\Vert\mathbf{P}^T\mathbf{P}-\mathbf{Q}^T\mathbf{Q}\Vert_F^2\\&\le\frac{3\sigma_1^*}{256}\Vert\mathbf{P}^T\mathbf{P}-\mathbf{Q}^T\mathbf{Q}\Vert_F^2.
		\end{split}
	\end{equation}
	From \textbf{Lemma \ref{t24}}, we have
	\begin{equation}
		\Vert\nabla_\mathbf{L}\mathcal{H}(\mathbf{P},\mathbf{Q})\Vert_F\le(1+\sqrt{\frac{2\alpha}{\min(10\alpha,0.1)}})\Vert\mathbf{L}-\mathbf{L}^*\Vert_F,
	\end{equation}
	so
	\begin{equation}
		\Vert\nabla_\mathbf{L}\mathcal{H}(\mathbf{P},\mathbf{Q})\Vert_F^2\le(1+\sqrt{\frac{2\alpha}{\min(10\alpha,0.1)}})^2\Vert\mathbf{L}-\mathbf{L}^*\Vert_F^2.
	\end{equation}
\end{proof}
\noindent We conduct the proof of \textbf{Theorem 2} by induction.\\\\
If $\alpha$ is small, then from \textbf{Theorem 1} we have $\Vert\mathbf{U}_0\mathbf{\Sigma}_0\mathbf{V}_0^T-\mathbf{L}^*\Vert_2\le\frac{1}{2}\sigma_1^*$. By Weyl's theorem, we have
\begin{equation}
	\Vert\mathbf{U}_0\mathbf{\Sigma}_0^\frac{1}{2}\Vert_2\le\sqrt{\frac{3\sigma_1^*}{2}},
\end{equation}
\begin{equation}
	\Vert\mathbf{V}_0\mathbf{\Sigma}_0^\frac{1}{2}\Vert_2\le\sqrt{\frac{3\sigma_1^*}{2}},
\end{equation}
and
\begin{equation}
	\Vert\mathbf{P}_0\Vert_2\le\Vert\mathbf{X}^T\mathbf{U}_0\mathbf{\Sigma}_0^\frac{1}{2}\Vert_2\le\Vert\mathbf{X}\Vert_2\Vert\mathbf{U}_0\mathbf{\Sigma}_0^\frac{1}{2}\Vert_2\le\sqrt{\frac{3\sigma_1^*}{2}},
\end{equation}
\begin{equation}
	\Vert\mathbf{Q}_0\Vert_2\le\Vert\mathbf{Y}^T\mathbf{V}_0\mathbf{\Sigma}_0^\frac{1}{2}\Vert_2\le\Vert\mathbf{Y}\Vert_2\Vert\mathbf{V}_0\mathbf{\Sigma}_0^\frac{1}{2}\Vert_2\le\sqrt{\frac{3\sigma_1^*}{2}}.
\end{equation}
In case (i), we thus have
\begin{equation}
	\Vert\mathbf{X}\mathbf{\Pi}_\mathcal{P}(\mathbf{P}_0)\Vert_{2,\infty}\le\sqrt\frac{2\mu_1r}{n}\Vert\mathbf{P}_0\Vert_2\le\sqrt\frac{3\sigma_1^*\mu_1r}{n},
\end{equation}
\begin{equation}
	\Vert\mathbf{Y}\mathbf{\Pi}_\mathcal{Q}(\mathbf{Q}_0)\Vert_{2,\infty}\le\sqrt\frac{2\mu_1r}{n}\Vert\mathbf{Q}_0\Vert_2\le\sqrt\frac{3\sigma_1^*\mu_1r}{n}.
\end{equation}
And it also follows that $d(\mathbf{\Pi}_\mathcal{P}(\mathbf{P}_t),\mathbf{\Pi}_\mathcal{Q}(\mathbf{Q}_t),\mathbf{P}^*,\mathbf{Q}^*)\le d(\mathbf{P}_t,\mathbf{Q}_t,\mathbf{P}^*,\mathbf{Q}^*)$.\\\\
By definition,
\begin{equation}
	\begin{split}
		\Vert P-P^\dagger\Vert_F^2&\le\delta,\\
		\Vert P-P^\dagger\Vert_2&\le\Vert P-P^\dagger\Vert_F\le\delta^\frac{1}{2}.\\
	\end{split}
\end{equation}
And from Weyl's theorem, if $\delta^\frac{1}{2}\le(\sqrt\frac{3}{2}-1)\sqrt{\sigma_1^*}$, we have
\begin{equation}
	\Vert P\Vert_2\le\sqrt\frac{3\sigma_1^*}{2}.
\end{equation}
Similarly, we also have
\begin{equation}
	\Vert Q\Vert_2\le\sqrt\frac{3\sigma_1^*}{2}.
\end{equation}
In cases (ii) and (iii), we have
\begin{equation}
	\Vert\mathbf{X}\mathbf{P}\Vert_{2,\infty}\le\Vert\mathbf{P}\Vert_2\Vert\mathbf{X}\Vert_{2,\infty}\le\sqrt{\frac{3\sigma_1^*}{2}}\times\sqrt{\frac{\mu_2 d}{n}}\le\sqrt{\frac{3d\mu\sigma_1^*}{2n}},
\end{equation}
\begin{equation}
	\Vert\mathbf{Y}\mathbf{Q}\Vert_{2,\infty}\le\Vert\mathbf{Q}\Vert_2\Vert\mathbf{Y}\Vert_{2,\infty}\le\sqrt{\frac{3\sigma_1^*}{2}}\times\sqrt{\frac{\mu_2 d}{n}}\le\sqrt{\frac{3d\mu\sigma_1^*}{2n}}.
\end{equation}
Now, we verify that $\delta$ diminishes:
\begin{equation}
	\begin{split}
	\delta_{t+1}&=\Vert\mathbf{P}_{t+1}-\mathbf{P}^\dagger_{t+1}\Vert_F^2+\Vert\mathbf{Q}_{t+1}-\mathbf{Q}^\dagger_{t+1}\Vert_F^2\\&\le\Vert\mathbf{P}_{t+1}-\mathbf{P}^\dagger_t\Vert_F^2+\Vert\mathbf{Q}_{t+1}-\mathbf{Q}^\dagger_t\Vert_F^2\\&=\Vert\mathbf{P}_t-\eta\nabla_\mathbf{P}\mathcal{H}_t-\eta\nabla_\mathbf{P}\mathcal{G}_t-\mathbf{P}^\dagger_t\Vert_F^2+\Vert\mathbf{Q}_t-\eta\nabla_\mathbf{Q}\mathcal{H}_t-\eta\nabla_\mathbf{Q}\mathcal{G}_t-\mathbf{Q}^\dagger_t\Vert_F^2\\&=\delta_t-2\eta\langle\nabla_\mathbf{P}\mathcal{H}_t+\nabla_\mathbf{P}\mathcal{G}_t,\mathbf{P}_t-\mathbf{P}^\dagger_t\rangle-2\eta\langle\nabla_\mathbf{Q}\mathcal{H}_t+\nabla_\mathbf{Q}\mathcal{G}_t,\mathbf{Q}_t-\mathbf{Q}^\dagger_t\rangle\\&\qquad+\eta^2\Vert\nabla_\mathbf{P}\mathcal{H}_t+\nabla_\mathbf{P}\mathcal{G}_t\Vert_F^2+\eta^2\Vert\nabla_\mathbf{Q}\mathcal{H}_t+\nabla_\mathbf{Q}\mathcal{G}_t\Vert_F^2\\&=\delta_t+\eta^2\Vert\nabla_\mathbf{P}\mathcal{H}_t+\nabla_\mathbf{P}\mathcal{G}_t\Vert_F^2+\eta^2\Vert\nabla_\mathbf{Q}\mathcal{H}_t+\nabla_\mathbf{Q}\mathcal{G}_t\Vert_F^2-2\eta\langle\nabla_\mathbf{P}\mathcal{G}_t,\mathbf{P}_t-\mathbf{P}^\dagger_t\rangle\\&\qquad-2\eta\langle\nabla_\mathbf{Q}\mathcal{G}_t,\mathbf{Q}_t-\mathbf{Q}^\dagger_t\rangle-2\eta\langle\nabla_\mathbf{L}\mathcal{H}_t,\mathbf{X}(\mathbf{P}_t\mathbf{Q}_t^T-\mathbf{P}^\dagger_t\mathbf{Q}^{\dagger T}_t+\Delta\mathbf{P}_t\Delta\mathbf{Q}_t^T)\mathbf{Y}^T\rangle\\&\le\delta_t-2\eta\langle\nabla_\mathbf{P}\mathcal{G}_t,\mathbf{P}_t-\mathbf{P}^\dagger_t\rangle-2\eta\langle\nabla_\mathbf{Q}\mathcal{G}_t,\mathbf{Q}_t-\mathbf{Q}^\dagger_t\rangle\\&\qquad-2\eta\langle\nabla_\mathbf{L}\mathcal{H}_t,\mathbf{X}(\mathbf{P}_t\mathbf{Q}_t^T-\mathbf{P}^\dagger_t\mathbf{Q}^{\dagger T}_t+\Delta\mathbf{P}_t\Delta\mathbf{Q}_t^T)\mathbf{Y}^T\rangle\\&\qquad+\eta^2(\Vert\nabla_\mathbf{P}\mathcal{H}_t\Vert_F+\Vert\nabla_\mathbf{P}\mathcal{G}_t\Vert_F)^2+\eta^2(\Vert\nabla_\mathbf{Q}\mathcal{H}_t\Vert_F+\Vert\nabla_\mathbf{Q}\mathcal{G}_t\Vert_F)^2\\&\le\delta_t-2\eta\langle\nabla_\mathbf{P}\mathcal{G}_t,\mathbf{P}_t-\mathbf{P}^\dagger_t\rangle-2\eta\langle\nabla_\mathbf{Q}\mathcal{G}_t,\mathbf{Q}_t-\mathbf{Q}^\dagger_t\rangle\\&\qquad-2\eta\langle\nabla_\mathbf{L}\mathcal{H}_t,\mathbf{X}(\mathbf{P}_t\mathbf{Q}_t^T-\mathbf{P}^\dagger_t\mathbf{Q}^{\dagger T}_t+\Delta\mathbf{P}_t\Delta\mathbf{Q}_t^T)\mathbf{Y}^T\rangle\\&\qquad+2\eta^2(\Vert\mathbf{X}^T\nabla_\mathbf{L}\mathcal{H}_t\mathbf{Y}\mathbf{Q}\Vert_F^2+\Vert\nabla_\mathbf{P}\mathcal{G}_t\Vert_F^2+\Vert(\mathbf{X}^T\nabla_\mathbf{L}\mathcal{H}_t\mathbf{Y})^T\mathbf{P}\Vert_F^2+\Vert\nabla_\mathbf{Q}\mathcal{G}_t\Vert_F^2)\\&\le\delta_t-2\eta\langle\nabla_\mathbf{P}\mathcal{G}_t,\mathbf{P}_t-\mathbf{P}^\dagger_t\rangle-2\eta\langle\nabla_\mathbf{Q}\mathcal{G}_t,\mathbf{Q}_t-\mathbf{Q}^\dagger_t\rangle\\&\qquad-2\eta\langle\nabla_\mathbf{L}\mathcal{H}_t,\mathbf{X}(\mathbf{P}_t\mathbf{Q}_t^T-\mathbf{P}^\dagger_t\mathbf{Q}^{\dagger T}_t+\Delta\mathbf{P}_t\Delta\mathbf{Q}_t^T)\mathbf{Y}^T\rangle\\&\qquad+2\eta^2(\Vert\nabla_\mathbf{Q}\mathcal{G}_t\Vert_F^2+\Vert\nabla_\mathbf{P}\mathcal{G}_t\Vert_F^2)\\&\qquad+2\eta^2(\Vert\mathbf{Q}\Vert_2^2\Vert\mathbf{X}^T\nabla_\mathbf{L}\mathcal{H}_t\mathbf{Y}\Vert_F^2+\Vert\mathbf{P}\Vert_2^2\Vert(\mathbf{X}^T\nabla_\mathbf{L}\mathcal{H}_t\mathbf{Y})^T\Vert_F^2)\\&\le\delta_t-2\eta\langle\nabla_\mathbf{P}\mathcal{G}_t,\mathbf{P}_t-\mathbf{P}^\dagger_t\rangle-2\eta\langle\nabla_\mathbf{Q}\mathcal{G}_t,\mathbf{Q}_t-\mathbf{Q}^\dagger_t\rangle\\&\qquad-2\eta\langle\nabla_\mathbf{L}\mathcal{H}_t,\mathbf{X}(\mathbf{P}_t\mathbf{Q}_t^T-\mathbf{P}^\dagger_t\mathbf{Q}^{\dagger T}_t+\Delta\mathbf{P}_t\Delta\mathbf{Q}_t^T)\mathbf{Y}^T\rangle\\&\qquad+2\eta^2(\Vert\nabla_\mathbf{Q}\mathcal{G}_t\Vert_F^2+\Vert\nabla_\mathbf{P}\mathcal{G}_t\Vert_F^2)\\&\qquad+2\eta^2(\Vert\mathbf{Q}\Vert_2^2\Vert\mathbf{X}\Vert_2^2\Vert\mathbf{Y}\Vert_2^2\Vert\nabla_\mathbf{L}\mathcal{H}_t\Vert_F^2+\Vert\mathbf{P}\Vert_2^2\Vert\mathbf{X}\Vert_2^2\Vert\mathbf{Y}\Vert_2^2\Vert\nabla_\mathbf{L}\mathcal{H}_t\Vert_F^2)\\&\le\delta_t-2\eta\langle\nabla_\mathbf{P}\mathcal{G}_t,\mathbf{P}_t-\mathbf{P}^\dagger_t\rangle-2\eta\langle\nabla_\mathbf{Q}\mathcal{G}_t,\mathbf{Q}_t-\mathbf{Q}^\dagger_t\rangle\\&\qquad-2\eta\langle\nabla_\mathbf{L}\mathcal{H}_t,\mathbf{X}(\mathbf{P}_t\mathbf{Q}_t^T-\mathbf{P}^\dagger_t\mathbf{Q}^{\dagger T}_t+\Delta\mathbf{P}_t\Delta\mathbf{Q}_t^T)\mathbf{Y}^T\rangle\\&\qquad+2\eta^2(\Vert\nabla_\mathbf{Q}\mathcal{G}_t\Vert_F^2+\Vert\nabla_\mathbf{P}\mathcal{G}_t\Vert_F^2)\\&\qquad+2\eta^2(\frac{3\sigma_1^*}{2}\Vert\nabla_\mathbf{L}\mathcal{H}_t\Vert_F^2+\frac{3\sigma_1^*}{2}\Vert\nabla_\mathbf{L}\mathcal{H}_t\Vert_F^2)\\&\le\delta_t-2\eta\langle\nabla_\mathbf{P}\mathcal{G}_t,\mathbf{P}_t-\mathbf{P}^\dagger_t\rangle-2\eta\langle\nabla_\mathbf{Q}\mathcal{G}_t,\mathbf{Q}_t-\mathbf{Q}^\dagger_t\rangle\\&\qquad-2\eta\langle\nabla_\mathbf{L}\mathcal{H}_t,\mathbf{X}(\mathbf{P}_t\mathbf{Q}_t^T-\mathbf{P}^\dagger_t\mathbf{Q}^{\dagger T}_t+\Delta\mathbf{P}_t\Delta\mathbf{Q}_t^T)\mathbf{Y}^T\rangle\\&\qquad+2\eta^2(\Vert\nabla_\mathbf{Q}\mathcal{G}_t\Vert_F^2+\Vert\nabla_\mathbf{P}\mathcal{G}_t\Vert_F^2+3\sigma_1^*\Vert\nabla_\mathbf{L}\mathcal{H}_t\Vert_F^2).
	\end{split}
\end{equation}
Applying \textbf{Lemma \ref{t26}}, we get
\begin{equation}
	\begin{split}
		&\delta_{t+1}\le\delta_t-2\eta\langle\nabla_\mathbf{P}\mathcal{G}_t,\mathbf{P}_t-\mathbf{P}_t^\dagger\rangle-2\eta\langle\nabla_\mathbf{Q}\mathcal{G}_t,\mathbf{Q}_t-\mathbf{Q}_t^\dagger\rangle\\&\qquad\qquad -2\eta\langle\nabla_\mathbf{L}\mathcal{H}_t,\mathbf{X}(\mathbf{P}_t\mathbf{Q}_t^T-\mathbf{P}_t^\dagger\mathbf{Q}_t^{\dagger T}+\Delta\mathbf{P}_t\Delta\mathbf{Q}_t^T)\mathbf{Y}^T\rangle\\&\qquad\qquad +\eta^2(\frac{3\sigma_1^*}{128}\Vert\mathbf{P}_t^T\mathbf{P}_t-\mathbf{Q}_t^T\mathbf{Q}_t\Vert_F^2+6(1+\sqrt{\frac{2\alpha}{\min(10\alpha,0.1)}})^2\sigma_1^*\Vert\mathbf{L}_t-\mathbf{L}_t^*\Vert_F^2).
	\end{split}
\end{equation}
Applying \textbf{Lemma \ref{t25}}, we get
\begin{equation}
	\begin{split}
		&\delta_{t+1}\le\delta_t+\eta(\frac{1}{8}\Vert\mathbf{L}_t-\mathbf{L}_t^*\Vert_F^2-\frac{1}{32}\Vert\mathbf{P}_t^T\mathbf{P}_t-\mathbf{Q}_t^T\mathbf{Q}_t\Vert_F^2)\\&\qquad\qquad +\eta(\frac{\sqrt{2}+\sqrt{3}}{16}\sqrt{\sigma_1^*\delta_t^3}-\frac{1}{32}(2\sqrt{\sigma_r^*\delta_t} - \delta_t)^2)\\&\qquad\qquad -2\eta\langle\nabla_\mathbf{L}\mathcal{H}_t,\mathbf{X}(\mathbf{P}_t\mathbf{Q}_t^T-\mathbf{P}_t^\dagger\mathbf{Q}_t^{\dagger T}+\Delta\mathbf{P}_t\Delta\mathbf{Q}_t^T)\mathbf{Y}^T\rangle\\&\qquad\qquad +\eta^2(\frac{3\sigma_1^*}{128}\Vert\mathbf{P}_t^T\mathbf{P}_t-\mathbf{Q}_t^T\mathbf{Q}_t\Vert_F^2+6(1+\sqrt{\frac{2\alpha}{\min(10\alpha,0.1)}})^2\sigma_1^*\Vert\mathbf{L}_t-\mathbf{L}_t^*\Vert_F^2).
	\end{split}
\end{equation}
Applying \textbf{Lemma \ref{t24}}, we have in case (i)
\begin{equation}
	\begin{split}
		&\quad\delta_{t+1}\\&\le\delta_t-\eta(\frac{1}{32}\Vert\mathbf{P}_t^T\mathbf{P}_t-\mathbf{Q}_t^T\mathbf{Q}_t\Vert_F^2+\frac{15}{8}\Vert\mathbf{L}_t-\mathbf{L}_t^*\Vert_F^2)\\&\qquad+\eta(\frac{\sqrt{2}+\sqrt{3}}{16}\sqrt{\sigma_1^*\delta_t^3}-\frac{1}{32}(2\sqrt{\sigma_r^*\delta_t} - \delta_t)^2+(\sqrt{2}+2\sqrt{\frac{\alpha}{\min(10\alpha,0.1)}})\sqrt{\delta_t^3}(\sqrt{\sigma_1^*}+\frac{\sqrt{2\delta_t}}{4}))\\&\qquad+\eta(\frac{\mu_1r\sigma_1^*\delta_t}{2}((4+\beta)\alpha+2\min(10\alpha,0.1))(3+\sqrt\frac{3}{2})^2+\frac{4\alpha\delta_t}{\beta\min(10\alpha,0.1)}(\sqrt{\sigma_1^*}+\frac{\sqrt{2\delta_t}}{4})^2)\\&\qquad+\eta^2(\frac{3\sigma_1^*}{128}\Vert\mathbf{P}_t^T\mathbf{P}_t-\mathbf{Q}_t^T\mathbf{Q}_t\Vert_F^2+6(1+\sqrt{\frac{2\alpha}{\min(10\alpha,0.1)}})^2\sigma_1^*\Vert\mathbf{L}_t-\mathbf{L}_t^*\Vert_F^2),
	\end{split}
\end{equation}
and in cases (ii) and (iii)
\begin{equation}
	\begin{split}
		&\quad\delta_{t+1}\\&\le\delta_t-\eta(\frac{1}{32}\Vert\mathbf{P}_t^T\mathbf{P}_t-\mathbf{Q}_t^T\mathbf{Q}_t\Vert_F^2+\frac{15}{8}\Vert\mathbf{L}_t-\mathbf{L}_t^*\Vert_F^2)\\&\qquad+\eta(\frac{\sqrt{2}+\sqrt{3}}{16}\sqrt{\sigma_1^*\delta_t^3}-\frac{1}{32}(2\sqrt{\sigma_r^*\delta_t} - \delta_t)^2+(\sqrt{2}+2\sqrt{\frac{\alpha}{\min(10\alpha,0.1)}})\sqrt{\delta_t^3}(\sqrt{\sigma_1^*}+\frac{\sqrt{2\delta_t}}{4}))\\&\qquad+\eta(\frac{\mu_2d\sigma_1^*\delta_t}{2}((4+\beta)\alpha+2\min(10\alpha,0.1))(3+\sqrt\frac{3}{2})^2+\frac{4\alpha\delta_t}{\beta\min(10\alpha,0.1)}(\sqrt{\sigma_1^*}+\frac{\sqrt{2\delta_t}}{4})^2)\\&\qquad+\eta^2(\frac{3\sigma_1^*}{128}\Vert\mathbf{P}_t^T\mathbf{P}_t-\mathbf{Q}_t^T\mathbf{Q}_t\Vert_F^2+6(1+\sqrt{\frac{2\alpha}{\min(10\alpha,0.1)}})^2\sigma_1^*\Vert\mathbf{L}_t-\mathbf{L}_t^*\Vert_F^2),
	\end{split}
\end{equation}
If $10\alpha<0.1$, then min$(10\alpha, 0.1)=10\alpha$.\\Therefore, we have in case (i)
\begin{equation}
    \begin{split}
        &\quad\delta_{t+1}\\&\le\delta_t-\eta(\frac{1}{32}\Vert\mathbf{P}_t^T\mathbf{P}_t-\mathbf{Q}_t^T\mathbf{Q}_t\Vert_F^2+\frac{15}{8}\Vert\mathbf{L}_t-\mathbf{L}_t^*\Vert_F^2)\\&\qquad+\eta(\frac{\sqrt{2}+\sqrt{3}}{16}\sqrt{\sigma_1^*\delta_t^3}-\frac{1}{32}(2\sqrt{\sigma_r^*\delta_t} - \delta_t)^2+(\sqrt{2}+\sqrt{\frac{2}{5}})\sqrt{\delta_t^3}(\sqrt{\sigma_1^*}+\frac{\sqrt{2\delta_t}}{4}))\\&\qquad+\eta(\frac{\mu_1r\sigma_1^*\delta_t}{2}((4+\beta)\alpha+20\alpha)(3+\sqrt\frac{3}{2})^2+\frac{2\delta_t}{5\beta}(\sqrt{\sigma_1^*}+\frac{\sqrt{2\delta_t}}{4})^2)\\&\qquad+\eta^2(\frac{3\sigma_1^*}{128}\Vert\mathbf{P}_t^T\mathbf{P}_t-\mathbf{Q}_t^T\mathbf{Q}_t\Vert_F^2+6(1+\sqrt{\frac{1}{5}})^2\sigma_1^*\Vert\mathbf{L}_t-\mathbf{L}_t^*\Vert_F^2)\\&\le\delta_t-\eta(\frac{1}{32}\Vert\mathbf{P}_t^T\mathbf{P}_t-\mathbf{Q}_t^T\mathbf{Q}_t\Vert_F^2+\frac{15}{8}\Vert\mathbf{L}_t-\mathbf{L}_t^*\Vert_F^2)\\&\qquad+\eta(\frac{\sqrt{2}+\sqrt{3}}{16}\sqrt{\sigma_1^*\delta_t^3}-\frac{1}{32}(2\sqrt{\sigma_r^*\delta_t} - \delta_t)^2+(\sqrt{2}+2\sqrt{10})\sqrt{\delta_t^3}(\sqrt{\sigma_1^*}+\frac{\sqrt{2\delta_t}}{4}))\\&\qquad+\eta(\frac{\mu_1r\sigma_1^*\delta_t}{2}((24+\beta)\alpha)(3+\sqrt\frac{3}{2})^2+\frac{40\delta_t}{\beta}(\sqrt{\sigma_1^*}+\frac{\sqrt{2\delta_t}}{4})^2)\\&\qquad+\eta^2(\frac{3\sigma_1^*}{128}\Vert\mathbf{P}_t^T\mathbf{P}_t-\mathbf{Q}_t^T\mathbf{Q}_t\Vert_F^2+6(1+\sqrt{20})^2\sigma_1^*\Vert\mathbf{L}_t-\mathbf{L}_t^*\Vert_F^2),
    \end{split}
\end{equation}
and in cases (ii) and (iii)
\begin{equation}
    \begin{split}
        &\quad\delta_{t+1}\\&\le\delta_t-\eta(\frac{1}{32}\Vert\mathbf{P}_t^T\mathbf{P}_t-\mathbf{Q}_t^T\mathbf{Q}_t\Vert_F^2+\frac{15}{8}\Vert\mathbf{L}_t-\mathbf{L}_t^*\Vert_F^2)\\&\qquad+\eta(\frac{\sqrt{2}+\sqrt{3}}{16}\sqrt{\sigma_1^*\delta_t^3}-\frac{1}{32}(2\sqrt{\sigma_r^*\delta_t} - \delta_t)^2+(\sqrt{2}+2\sqrt{10})\sqrt{\delta_t^3}(\sqrt{\sigma_1^*}+\frac{\sqrt{2\delta_t}}{4}))\\&\qquad+\eta(\frac{\mu_2d\sigma_1^*\delta_t}{2}((24+\beta)\alpha)(3+\sqrt\frac{3}{2})^2+\frac{40\delta_t}{\beta}(\sqrt{\sigma_1^*}+\frac{\sqrt{2\delta_t}}{4})^2)\\&\qquad+\eta^2(\frac{3\sigma_1^*}{128}\Vert\mathbf{P}_t^T\mathbf{P}_t-\mathbf{Q}_t^T\mathbf{Q}_t\Vert_F^2+6(1+\sqrt{20})^2\sigma_1^*\Vert\mathbf{L}_t-\mathbf{L}_t^*\Vert_F^2).
    \end{split}
\end{equation}
On the other hand, we have min$(10\alpha,0.1)=0.1$ if $10\alpha\ge0.1$.\\Then, we have in case (i)
\begin{equation}
    \begin{split}
        &\quad\delta_{t+1}\\&\le\delta_t-\eta(\frac{1}{32}\Vert\mathbf{P}_t^T\mathbf{P}_t-\mathbf{Q}_t^T\mathbf{Q}_t\Vert_F^2+\frac{15}{8}\Vert\mathbf{L}_t-\mathbf{L}_t^*\Vert_F^2)\\&\qquad+\eta(\frac{\sqrt{2}+\sqrt{3}}{16}\sqrt{\sigma_1^*\delta_t^3}-\frac{1}{32}(2\sqrt{\sigma_r^*\delta_t} - \delta_t)^2+(\sqrt{2}+2\sqrt{10\alpha})\sqrt{\delta_t^3}(\sqrt{\sigma_1^*}+\frac{\sqrt{2\delta_t}}{4}))\\&\qquad+\eta(\frac{\mu_1r\sigma_1^*\delta_t}{2}((4+\beta)\alpha+0.2)(3+\sqrt\frac{3}{2})^2+\frac{40\alpha\delta_t}{\beta}(\sqrt{\sigma_1^*}+\frac{\sqrt{2\delta_t}}{4})^2)\\&\qquad+\eta^2(\frac{3\sigma_1^*}{128}\Vert\mathbf{P}_t^T\mathbf{P}_t-\mathbf{Q}_t^T\mathbf{Q}_t\Vert_F^2+6(1+\sqrt{20\alpha})^2\sigma_1^*\Vert\mathbf{L}_t-\mathbf{L}_t^*\Vert_F^2).
    \end{split}
\end{equation}
But $\alpha\le1$, so
\begin{equation}
    \begin{split}
        &\quad\delta_{t+1}\\&\le\delta_t-\eta(\frac{1}{32}\Vert\mathbf{P}_t^T\mathbf{P}_t-\mathbf{Q}_t^T\mathbf{Q}_t\Vert_F^2+\frac{15}{8}\Vert\mathbf{L}_t-\mathbf{L}_t^*\Vert_F^2)\\&\qquad+\eta(\frac{\sqrt{2}+\sqrt{3}}{16}\sqrt{\sigma_1^*\delta_t^3}-\frac{1}{32}(2\sqrt{\sigma_r^*\delta_t} - \delta_t)^2+(\sqrt{2}+2\sqrt{10})\sqrt{\delta_t^3}(\sqrt{\sigma_1^*}+\frac{\sqrt{2\delta_t}}{4}))\\&\qquad+\eta(\frac{\mu_1r\sigma_1^*\delta_t}{2}((4+\beta)\alpha+20\alpha)(3+\sqrt\frac{3}{2})^2+\frac{40\delta_t}{\beta}(\sqrt{\sigma_1^*}+\frac{\sqrt{2\delta_t}}{4})^2)\\&\qquad+\eta^2(\frac{3\sigma_1^*}{128}\Vert\mathbf{P}_t^T\mathbf{P}_t-\mathbf{Q}_t^T\mathbf{Q}_t\Vert_F^2+6(1+\sqrt{20})^2\sigma_1^*\Vert\mathbf{L}_t-\mathbf{L}_t^*\Vert_F^2).
    \end{split}
\end{equation}
And, similarly, we have in cases (ii) and (iii)
\begin{equation}
    \begin{split}
        &\quad\delta_{t+1}\\&\le\delta_t-\eta(\frac{1}{32}\Vert\mathbf{P}_t^T\mathbf{P}_t-\mathbf{Q}_t^T\mathbf{Q}_t\Vert_F^2+\frac{15}{8}\Vert\mathbf{L}_t-\mathbf{L}_t^*\Vert_F^2)\\&\qquad+\eta(\frac{\sqrt{2}+\sqrt{3}}{16}\sqrt{\sigma_1^*\delta_t^3}-\frac{1}{32}(2\sqrt{\sigma_r^*\delta_t} - \delta_t)^2+(\sqrt{2}+2\sqrt{10})\sqrt{\delta_t^3}(\sqrt{\sigma_1^*}+\frac{\sqrt{2\delta_t}}{4}))\\&\qquad+\eta(\frac{\mu_2d\sigma_1^*\delta_t}{2}((4+\beta)\alpha+20\alpha)(3+\sqrt\frac{3}{2})^2+\frac{40\delta_t}{\beta}(\sqrt{\sigma_1^*}+\frac{\sqrt{2\delta_t}}{4})^2)\\&\qquad+\eta^2(\frac{3\sigma_1^*}{128}\Vert\mathbf{P}_t^T\mathbf{P}_t-\mathbf{Q}_t^T\mathbf{Q}_t\Vert_F^2+6(1+\sqrt{20})^2\sigma_1^*\Vert\mathbf{L}_t-\mathbf{L}_t^*\Vert_F^2).
    \end{split}
\end{equation}
If $\eta\le\frac{5}{16(1+\sqrt{20})^2\sigma_1^*}$, we have in case (i)
\begin{multline}
    \delta_{t+1}\le\delta_t+\eta(\frac{\mu_1r\alpha\sigma_1^*\delta_t}{2}(24+\beta)(3+\sqrt\frac{3}{2})^2+\frac{40\delta_t}{\beta}(\sqrt{\sigma_1^*}+\frac{\sqrt{2\delta_t}}{4})^2)\\+\eta(\frac{\sqrt{2}+\sqrt{3}}{16}\sqrt{\sigma_1^*\delta_t^3}-\frac{1}{32}(2\sqrt{\sigma_r^*\delta_t} - \delta_t)^2+(\sqrt{2}+2\sqrt{10})\sqrt{\delta_t^3}(\sqrt{\sigma_1^*}+\frac{\sqrt{2\delta_t}}{4})).
\end{multline}
and in cases (ii) and (iii)
\begin{multline}
    \delta_{t+1}\le\delta_t+\eta(\frac{\mu_2d\alpha\sigma_1^*\delta_t}{2}(24+\beta)(3+\sqrt\frac{3}{2})^2+\frac{40\delta_t}{\beta}(\sqrt{\sigma_1^*}+\frac{\sqrt{2\delta_t}}{4})^2)\\+\eta(\frac{\sqrt{2}+\sqrt{3}}{16}\sqrt{\sigma_1^*\delta_t^3}-\frac{1}{32}(2\sqrt{\sigma_r^*\delta_t} - \delta_t)^2+(\sqrt{2}+2\sqrt{10})\sqrt{\delta_t^3}(\sqrt{\sigma_1^*}+\frac{\sqrt{2\delta_t}}{4})).
\end{multline}
If $\delta_t\le2\sigma_r^*$,
we have in case (i)
\begin{equation}
	\begin{split}
	    &\delta_{t+1}\le\delta_t+\eta(\frac{25\sqrt{2}+\sqrt{3}+48\sqrt{10}}{16}\delta_t\sqrt{\sigma_1^*\delta_t}-\frac{3-\sqrt{2}}{16}\sigma_r^*\delta_t)\\&\qquad\qquad+\eta(\frac{\mu_1r\sigma_1^*\delta_t\alpha}{2}(24+\beta)(3+\sqrt\frac{3}{2})^2+\frac{90\delta_t\sigma_1^*}{\beta})\\&\qquad\le\delta_t(1+\eta(\frac{25\sqrt{2}+\sqrt{3}+48\sqrt{10}}{16}\sqrt{\sigma_1^*\delta_t}-\frac{3-\sqrt{2}}{16}\sigma_r^*)\\&\qquad\qquad+\eta(\frac{\mu_1r\sigma_1^*\alpha}{2}(24+\beta)(3+\sqrt\frac{3}{2})^2+\frac{90\sigma_1^*}{\beta})),
	\end{split}
\end{equation}
and in cases (ii) and (iii)
\begin{equation}
	\begin{split}
	    &\delta_{t+1}\le\delta_t(1+\eta(\frac{25\sqrt{2}+\sqrt{3}+48\sqrt{10}}{16}\sqrt{\sigma_1^*\delta_t}-\frac{3-\sqrt{2}}{16}\sigma_r^*)\\&\qquad\qquad+\eta(\frac{\mu_2d\sigma_1^*\alpha}{2}(24+\beta)(3+\sqrt\frac{3}{2})^2+\frac{90\sigma_1^*}{\beta})),
	\end{split}
\end{equation}
In case (i), if $\alpha\le\frac{1}{16\kappa r\mu_1}$, we have
\begin{equation}
	\delta^\frac{1}{2}=d(\mathbf{P}_0,\mathbf{Q}_0,\mathbf{P}^*,\mathbf{Q}^*)\le18\alpha r\mu_1\sqrt{r\kappa\sigma_1^*},
\end{equation}
which leads to
\begin{equation}
	\begin{split}
	    &\delta_{t+1}\le\delta_t(1+\eta(\frac{\mu_1r\sigma_1^*\alpha}{2}(24+\beta)(3+\sqrt\frac{3}{2})^2-\frac{3-\sqrt{2}}{16}\sigma_r^*)\\&\qquad\qquad+\eta(\frac{90\sigma_1^*}{\beta}+\frac{225\sqrt{2}+9\sqrt{3}+432\sqrt{10}}{8}\alpha r\mu_1\sigma_1^*\sqrt{r\kappa})),
	\end{split}
\end{equation}
In case (ii), if $\alpha\le\frac{1}{16\kappa d\mu_2}$, we have
\begin{equation}
	\delta^\frac{1}{2}=d(\mathbf{P}_0,\mathbf{Q}_0,\mathbf{P}^*,\mathbf{Q}^*)\le18\alpha d\mu_2\sqrt{r\kappa\sigma_1^*},
\end{equation}
which leads to
\begin{equation}
	\begin{split}
	    &\delta_{t+1}\le\delta_t(1+\eta(\frac{\mu_2d\sigma_1^*\alpha}{2}(24+\beta)(3+\sqrt\frac{3}{2})^2-\frac{3-\sqrt{2}}{16}\sigma_r^*)\\&\qquad\qquad+\eta(\frac{90\sigma_1^*}{\beta}+\frac{225\sqrt{2}+9\sqrt{3}+432\sqrt{10}}{8}\alpha d\mu_2\sigma_1^*\sqrt{r\kappa})),
	\end{split}
\end{equation}
In case (iii), if $\alpha\le\frac{1}{16\kappa r\mu_1}$, we have
\begin{equation}
	\delta^\frac{1}{2}=d(\mathbf{P}_0,\mathbf{Q}_0,\mathbf{P}^*,\mathbf{Q}^*)\le18\alpha r\mu_1\sqrt{r\kappa\sigma_1^*},
\end{equation}
which leads to
\begin{equation}
	\begin{split}
	    &\delta_{t+1}\le\delta_t(1+\eta(\frac{\mu_2d\sigma_1^*\alpha}{2}(24+\beta)(3+\sqrt\frac{3}{2})^2-\frac{3-\sqrt{2}}{16}\sigma_r^*)\\&\qquad\qquad+\eta(\frac{90\sigma_1^*}{\beta}+\frac{225\sqrt{2}+9\sqrt{3}+432\sqrt{10}}{8}\alpha r\mu_1\sigma_1^*\sqrt{r\kappa})),
	\end{split}
\end{equation}
In case (i), we require that
\begin{multline}
    \frac{\mu_1r\sigma_1^*\alpha}{2}(24+\beta)(3+\sqrt\frac{3}{2})^2-\frac{3-\sqrt{2}}{16}\sigma_r^*+\frac{90\sigma_1^*}{\beta}\\+\frac{225\sqrt{2}+9\sqrt{3}+432\sqrt{10}}{8}\alpha r\mu_1\sigma_1^*\sqrt{r\kappa}\le0,
\end{multline}
which leads to
\begin{equation}
	\alpha\le\frac{\frac{3-\sqrt{2}}{16}+\frac{90\kappa}{\beta}}{\frac{\mu_1r\kappa}{2}(24+\beta)(3+\sqrt\frac{3}{2})^2+\frac{225\sqrt{2}+9\sqrt{3}+432\sqrt{10}}{8}\kappa r\mu_1\sqrt{r\kappa}}.
\end{equation}
Since other constraints on $\alpha$ are milder, for $\beta$ large enough, there exist $c_1$ and $c_2$ such that if $\alpha\le\frac{c_1}{\mu_1(\kappa r)^\frac{3}{2}}$,
\begin{equation}
\delta_t\le(1-c_2\eta\sigma_r^*)^t\delta_0.
\end{equation}
In case (ii), we require that
\begin{multline}
    \frac{\mu_2d\sigma_1^*\alpha}{2}(24+\beta)(3+\sqrt\frac{3}{2})^2-\frac{3-\sqrt{2}}{16}\sigma_r^*+\frac{90\sigma_1^*}{\beta}\\+\frac{225\sqrt{2}+9\sqrt{3}+432\sqrt{10}}{8}\alpha d\mu_2\sigma_1^*\sqrt{r\kappa}\le0,
\end{multline}
which leads to
\begin{equation}
	\alpha\le\frac{\frac{3-\sqrt{2}}{16}+\frac{90\kappa}{\beta}}{\frac{\mu_2d\kappa}{2}(24+\beta)(3+\sqrt\frac{3}{2})^2+\frac{225\sqrt{2}+9\sqrt{3}+432\sqrt{10}}{8}\kappa d\mu_2\sqrt{r\kappa}}.
\end{equation}
Since other constraints on $\alpha$ are milder, for $\beta$ large enough, there exist $c_1$ and $c_2$ such that if $\alpha\le\frac{c_3}{\mu_2dr^\frac{1}{2}\kappa^\frac{3}{2}}$,
\begin{equation}
\delta_t\le(1-c_4\eta\sigma_r^*)^t\delta_0.
\end{equation}
In case (iii), we require that
\begin{multline}
    \frac{\mu_2d\sigma_1^*\alpha}{2}(24+\beta)(3+\sqrt\frac{3}{2})^2-\frac{3-\sqrt{2}}{16}\sigma_r^*+\frac{90\sigma_1^*}{\beta}\\+\frac{225\sqrt{2}+9\sqrt{3}+432\sqrt{10}}{8}\alpha r\mu_1\sigma_1^*\sqrt{r\kappa}\le0,
\end{multline}
which leads to
\begin{equation}
	\alpha\le\frac{\frac{3-\sqrt{2}}{16}+\frac{90\kappa}{\beta}}{\frac{\mu_2d\kappa}{2}(24+\beta)(3+\sqrt\frac{3}{2})^2+\frac{225\sqrt{2}+9\sqrt{3}+432\sqrt{10}}{8}\kappa r\mu_1\sqrt{r\kappa}}.
\end{equation}
Since other constraints on $\alpha$ are milder, for $\beta$ large enough, there exist $c_5$ and $c_6$ such that if $\alpha\le c_5\min(\frac{1}{\mu_2d\kappa},\frac{1}{\mu_1(\kappa r)^\frac{3}{2}})$,
\begin{equation}
\delta_t\le(1-c_6\eta\sigma_r^*)^t\delta_0.
\end{equation}
\end{appendices}

\bibliographystyle{plainnat}
\bibliography{egbib}

\begin{thebibliography}{25}
\providecommand{\natexlab}[1]{#1}
\providecommand{\url}[1]{\texttt{#1}}
\expandafter\ifx\csname urlstyle\endcsname\relax
  \providecommand{\doi}[1]{doi: #1}\else
  \providecommand{\doi}{doi: \begingroup \urlstyle{rm}\Url}\fi

\bibitem[Basri and Jacobs(2003)]{basri03}
R.~Basri and D.~W. Jacobs.
\newblock Lambertian reflectance and linear subspaces.
\newblock \emph{TPAMI}, 25:\penalty0 218--233, 2003.

\bibitem[Candes et~al.(2011)Candes, Li, Ma, and Wright]{candes11}
E.~J. Candes, X.~Li, Y.~Ma, and J.~Wright.
\newblock Robust principal component analysis?
\newblock \emph{Journal of the ACM}, 58:\penalty0 11:1--11:37, 2011.

\bibitem[Chandrasekaran et~al.(2011)Chandrasekaran, Sanghavi, Parrilo, and
  Willsky]{chandrasekaran11}
V.~Chandrasekaran, S.~Sanghavi, P.~A. Parrilo, and A.~S. Willsky.
\newblock Rank-sparsity incoherence for matrix decomposition.
\newblock \emph{SIAM J. Optim.}, 21:\penalty0 572--596, 2011.

\bibitem[Chandrasekarana and Jordan(2013)]{chandrasekarana13}
V.~Chandrasekarana and M.~I. Jordan.
\newblock Computational and statistical tradeoffs via convex relaxation.
\newblock \emph{PNAS}, 110:\penalty0 1181--1190, 2013.

\bibitem[Chiang et~al.(2015)Chiang, Hseih, and Dhillon]{chiang15}
K.~Chiang, C.~Hseih, and I.~Dhillon.
\newblock Matrix completion with noisy side information.
\newblock In \emph{NIPS}, 2015.

\bibitem[Chiang et~al.(2016)Chiang, Hsieh, and Dhillon]{chiang16}
K.~Chiang, C.~Hsieh, and I.~Dhillon.
\newblock Robust principal component analysis with side information.
\newblock In \emph{ICML}, 2016.

\bibitem[Ge et~al.(2016)Ge, Lee, and Ma]{ge16}
R.~Ge, J.~Lee, and T.~Ma.
\newblock Matrix completion has no spurious local minimum.
\newblock In \emph{NIPS}, 2016.

\bibitem[Gong et~al.(2013)Gong, Zhang, Lu, Huang, and Ye]{gong13}
P.~Gong, C.~Zhang, Z.~Lu, J.~Huang, and J.~Ye.
\newblock A general iterative shrinkage and thresholding algorithm for
  non-convex regularized optimization problems.
\newblock In \emph{ICML}, 2013.

\bibitem[Hsu et~al.(2011)Hsu, Kakade, and Zhang]{hsu11}
D.~Hsu, S.~M. Kakade, and T.~Zhang.
\newblock Robust matrix decomposition with sparse corruptions.
\newblock \emph{TIT}, 2011.

\bibitem[Kohler and Lucchi(2017)]{kohler17}
J.~M. Kohler and A.~Lucchi.
\newblock Sub-sampled cubic regularization for non-convex optimization.
\newblock In \emph{ICML}, 2017.

\bibitem[Liu et~al.(2010)Liu, Lin, and Yu]{liu10}
G.~Liu, Z.~Lin, and Y.~Yu.
\newblock Robust subspace segmentation by low-rank representation.
\newblock In \emph{ICML}, 2010.

\bibitem[Liu et~al.(2017)Liu, Liu, and Li]{liu17}
G.~Liu, Q.~Liu, and P.~Li.
\newblock Blessing of dimensionality: Recovering mixture data via dictionary
  pursuit.
\newblock \emph{TPAMI}, 39:\penalty0 47--60, 2017.

\bibitem[Netrapalli et~al.(2014)Netrapalli, N, Sanghavi, Anandkumar, and
  Jain]{netrapalli14}
P.~Netrapalli, N.~U. N, S.~Sanghavi, A.~Anandkumar, and P.~Jain.
\newblock Non-convex robust pca.
\newblock In \emph{NIPS}, 2014.

\bibitem[Niranjan et~al.(2017)Niranjan, Rajkumar, and Tulabandhula]{niranjan17}
U.N. Niranjan, A.~Rajkumar, and T.~Tulabandhula.
\newblock Provable inductive robust pca via iterative hard thresholding.
\newblock In \emph{UAI}, 2017.

\bibitem[Oh et~al.(2015)Oh, Tai, Bazin, Kim, and Kweon]{oh15}
T.~Oh, Y.~Tai, J.~Bazin, H.~Kim, and I.~S. Kweon.
\newblock Partial sum minimization of singular values in robust pca: Algorithm
  and applications.
\newblock \emph{TPAMI}, 38:\penalty0 744--758, 2015.

\bibitem[Reich and Zaslavski(2012)]{reich12}
S.~Reich and A.~Zaslavski.
\newblock \emph{Infinite Products of Operators and Their Applications}.
\newblock the AMS and Bar-Ilan University, 2012.

\bibitem[Sagonas et~al.(2014)Sagonas, Panagakis, Zafeiriou, and
  Pantic]{sagnoas14}
C.~Sagonas, Y.~Panagakis, S.~Zafeiriou, and M.~Pantic.
\newblock Raps: Robust and efficient automatic construction of person-specific
  deformable models.
\newblock In \emph{CVPR}, 2014.

\bibitem[Shang et~al.(2017)Shang, Cheng, Liu, Luo, and Lin]{shang17}
F.~Shang, J.~Cheng, Y.~Liu, Z.~Luo, and Z.~Lin.
\newblock Bilinear factor matrix norm minimization for robust pca: Algorithms
  and applications.
\newblock \emph{TPAMI}, PP, 2017.

\bibitem[Tu et~al.(2016)Tu, Boczar, Simchowitz, Soltanolkotabi, and
  Recht]{tu16}
S.~Tu, R.~Boczar, M.~Simchowitz, M.~Soltanolkotabi, and B.~Recht.
\newblock Low-rank solutions of linear matrix equations via procrustes flow.
\newblock In \emph{ICML}, 2016.

\bibitem[Wright et~al.(2009)Wright, Ganesh, Rao, and Ma]{wright09}
J.~Wright, A.~Ganesh, S.~Rao, and Y.~Ma.
\newblock Robust principal component analysis: Exact recovery of corrupted
  low-rank matrices via convex optimization.
\newblock In \emph{NIPS}, 2009.

\bibitem[Wright et~al.(2011)Wright, Ganesh, Yang, Zhou, and Ma]{wright11}
J.~Wright, A.~Ganesh, A.~Yang, Z.~Zhou, and Y.~Ma.
\newblock Sparsity and robustness in face recognition.
\newblock \emph{arXiv:1111.1014}, 2011.

\bibitem[Xiong et~al.(2016)Xiong, Liu, and Tao]{xiong16}
H.~Xiong, T.~Liu, and D.~Tao.
\newblock Diversified dynamical gaussian process latent variable model for
  video repair.
\newblock In \emph{AAAI}, 2016.

\bibitem[Xu et~al.(2013)Xu, Jin, and Zhou]{xu13}
M.~Xu, R.~Jin, and Z.~Zhou.
\newblock Speedup matrix completion with side information: application to
  multi-label learning.
\newblock In \emph{NIPS}, 2013.

\bibitem[Xue et~al.(2017)Xue, Panagakis, and Zafeiriou]{xue17}
N.~Xue, Y.~Panagakis, and S.~Zafeiriou.
\newblock Side information in robust principal component analysis: Algorithms
  and applications.
\newblock In \emph{ICCV}, 2017.

\bibitem[Yi et~al.(2016)Yi, Park, Chen, and Caramanis]{yi16}
X.~Yi, D.~Park, Y.~Chen, and C.~Caramanis.
\newblock Fast algorithms for robust pca via gradient descent.
\newblock In \emph{NIPS}, 2016.

\end{thebibliography}
\end{document}